%% file: ReStud.tex
\documentclass[12pt]{article}

\usepackage[margin=1in]{geometry}

\usepackage{setspace}

\usepackage{hyperref}
\hypersetup{colorlinks,
            linkcolor=blue,
            citecolor=blue,
            urlcolor=magenta,
            linktocpage,
            plainpages=false}
            
\usepackage{microtype}
\usepackage{graphicx}
\usepackage{subfigure}
\usepackage{booktabs} 
\usepackage{comment}
\usepackage{wrapfig}
\usepackage{arydshln}
\usepackage{natbib}
\usepackage{hyperref}

\usepackage{amsmath}
\usepackage{amssymb}
\usepackage{mathtools}
\usepackage{amsthm}

\usepackage{algorithm}
\usepackage{algorithmic}

\usepackage[capitalize,noabbrev]{cleveref}

\theoremstyle{plain}

\newtheorem{theorem}{Theorem}[section]
\newtheorem{lemma}[theorem]{Lemma}

\newtheorem{proposition}[theorem]{Proposition}

\newtheorem{definition}[theorem]{Definition}

\usepackage[textsize=tiny]{todonotes}

\DeclareMathOperator*{\argmax}{arg\,max}
\DeclareMathOperator*{\argmin}{arg\,min}

\renewcommand{\eqref}[1]{(\ref{#1})}

\usepackage{enumitem}
\usepackage{bm}
\usepackage{bbm}
\usepackage{multirow}
\usepackage{authblk}
\allowdisplaybreaks

\usepackage{xcolor}
\newcount\Comments  
\newcommand{\kibitz}[2]{\ifnum\Comments=1\textcolor{#1}{#2}\fi}

\newcommand*{\email}[1]{\texttt{#1}}

\usepackage{url}

\title{Generalized Neyman Allocation for Locally Minimax Optimal Best-Arm Identification}

\author{Masahiro Kato\thanks{12th Floor, Kojimachi Odori Building, 2-4-1 Kojimachi, Chiyoda-ku, Tokyo 102-0083, Japan. Phone: +81-80-4948-7748.}}

\affil{
Graduate School of Arts and Sciences, The University of Tokyo\\
Data Analytics Department, Mizuho–DL Financial Technology, Co., Ltd.\\
\email{mkato-csecon@g.ecc.u-tokyo.ac.jp}}

\date{First version:  May 2024,  This version is of  \today.
 \\ \indent JEL Classification: C18, C21, C44.}

\begin{document}

\maketitle

\begin{abstract}
This study investigates an asymptotically locally minimax optimal algorithm for fixed-budget best-arm identification (BAI). We propose the Generalized Neyman Allocation (GNA) algorithm and demonstrate that its worst-case probability of misidentifying the best arm aligns with the worst-case lower bound under the small-gap regime, where the gaps between the expected outcomes of the best and suboptimal arms are small. Our lower and upper bounds are tight, matching exactly—including constant terms—within the small-gap regime. The GNA algorithm generalizes the Neyman allocation for two-armed bandits \citep{Neyman1934OnTT, Kaufman2016complexity} and refines existing BAI algorithms, such as those proposed by \citet{glynn2004large}. By proposing an asymptotic minimax optimal algorithm under distributions restricted by the small-gap regime, we address the longstanding open issue in BAI \citep{kaufmann2020hdr} and treatment choice \citep{Kasy2021,Ariu2021}. 

\end{abstract}

\textbf{Keywords:} best-arm identification, adaptive experimental design, Neyman allocation

\clearpage

\input{Main.tex}

\end{document}

%% file: Main.tex
\section{Introduction}
We investigate a minimax optimal algorithm for the problem of \emph{fixed-budget best-arm identification} \citep[BAI;][]{Audibert2010,Bubeck2011}. Fixed-budget BAI is a specific instance of adaptive experimentation where, given multiple treatment arms and sample size (budget), we allocate these arms to experimental units during the experiment, aiming to identify the \emph{best arm}, the one with the highest expected outcome, at the end of the experiment.\footnote{The term ``arms'' is also referred to as ``treatments'' or ``arms'' in the literature.} \footnote{There is another setting called fixed-confidence BAI, in which we continue an adaptive experiment until we identify the best arm with a certain fixed probability. In this setting, the sample size is not fixed but is treated as a stopping time.} BAI is closely connected to both bandit problems and causal inference, and it has been extensively studied in various fields, including machine learning \citep{Kaufman2016complexity}, economics \citep{Kasy2021}, and operations research \citep{Shin2018}.

In this study, we propose an asymptotically minimax optimal algorithm whose performance upper bound matches the lower bound under the worst-case distribution as the budget approaches infinity and the differences between the expected outcomes of the best and suboptimal arms approach zero. Note that the differences (\emph{gaps}) correspond to the average treatment effects in causal inference. We call this setting the \emph{small-gap regime}. The existence of optimal algorithms in fixed-budget BAI has long been an open question \citep{Ariu2021}. We address this issue by introducing a novel minimax algorithm.

Our proposed algorithm generalizes the Neyman allocation \citep{Neyman1934OnTT}, a method that has gained significant attention in the field of adaptive experimental design \citep{Hahn2011,Meehan2018,Kato2020adaptive,adusumilli2022minimax,CAI2024105793}. The existing Neyman allocation is limited to binary arms, and their extension to the multi-armed case remains unexplored. In this study, we bridge this gap by presenting a generalized version of the Neyman allocation.

The remainder of this paper is structured as follows. In this section, we outline the problem setting, provide background information on BAI, discuss related work, and summarize our contributions. Section~\ref{sec:main_lower} develops a worst-case lower bound under the small-gap regime. In Section~\ref{sec:GNA_algorithm}, we introduce the Generalized Neyman Allocation (GNA) algorithm. Section~\ref{sec:upper_bound} derives an upper bound for the GNA algorithm and demonstrates its local asymptotic minimax optimality by proving that the lower and upper bounds converge as the budget approaches infinity and the gap between arm outcomes diminishes.

\subsection{Problem setting}
In our problem, we consider a decision-maker who conducts an adaptive experiment with a fixed budget (sample size) $T$ and a fixed set of arms $[K] \coloneqq \{1,2,\dots, K\}$. Each arm $a \in [K]$ has a potential random outcome $Y_a \in \mathcal{Y} \subset \mathbb{R}$, where $\mathcal{Y}$ is an outcome space \citep{Neyman1923,Rubin1974}. Let $P_a$ be the marginal distribution of $Y_a$ for each $a \in [K]$, and let $P \coloneqq (P_a)_{a\in[K]}$ be the set of $P_a$.\footnote{We can define $P$ as a joint distribution of $(Y_a)_{a\in[K]}$. However, since multiple outcomes cannot be observed simultaneously, both definitions lead to the same consequences in our analysis. Following the literature on bandit studies, such as \citep{Kaufman2016complexity}, we define $P$ as a set of marginal distributions.} Let $\mathbb{P}_{P}$ and $\mathbb{E}_{P}$ be the probability and expectation under $P$, respectively. Under $P$, let $a^\star(P) \coloneqq \argmax_{a \in [K]} \mu_a(P)$ be the best arm, where $\mu_{a}(P) \coloneqq  \mathbb{E}_{P}[Y_a]$ is the expected value of $Y_a$. Let $P_0$ be the distribution that generates data during an adaptive experiment, called the true distribution.

In each round $t \in [T] \coloneqq \{1,2,\dots, T\}$, 
\begin{enumerate}[topsep=0pt, itemsep=0pt, partopsep=0pt, leftmargin=*]
    \item Potential outcomes $(Y_{1, t}, Y_{2, t}, \dots, Y_{K, t})$ are generated from $P_0$;
    \item The decision-maker allocates arm $A_t \in [K]$ based on past observations $\{(A_s, Y_s)\}_{s=1}^{t-1}$;
    \item The decision-maker observes the corresponding outcome $Y_t$ linked to the allocated arm $A_t$ as $Y_t = \sum_{a \in [K]}\mathbbm{1}[A_t = a]Y_{a, t}$.
\end{enumerate}
At the end of the experiment, the decision-maker constructs an estimator $\widehat{a}_T \in [K]$ of the best arm $a^\star(P_0)$. The decision-maker's goal is to identify the arm with the highest expected outcome, minimizing the probability of misidentification $\mathbb{P}_{P_0}(\widehat{a}_T \neq a^\star(P_0))$ at the end of the experiment. 

We define an algorithm as a pair of $((A_t)_{t \in [T]}, \widehat{a}_T)$, where $(A_t)_{t \in [T]}$ is the allocation rule, and $\widehat{a}_T$ is the estimation rule. Formally, with the sigma-algebras $\mathcal{F}_{t} = \sigma(A_1, Y_1, \ldots, A_t, Y_t)$, an algorithm is a pair $((A_t)_{t \in [T]}, \widehat{a}_T)$, where 
\begin{itemize}[topsep=0pt, itemsep=0pt, partopsep=0pt, leftmargin=*]
\item $(A_t)_{t \in [T]}$ is an allocation rule, which is $\mathcal{F}_{t-1}$-measurable. Under this rule, the decision-maker allocates an arm $A_t \in [K]$ in each round $t$ using observations up to round $t-1$.
\item $\widehat{a}_T$ is an estimation rule, which is an $\mathcal{F}_T$-measurable estimator of the best arm $a^\star(P_0)$ using observations up to round $T$.
\end{itemize}
We denote an algorithm by $\pi$. We also denote $A_t$ and $\widehat{a}_T$ by $A^\pi_t$ and $\widehat{a}^\pi_T$ when we emphasize that $A_t$ and $\widehat{a}_T$ depend on $\pi$. 

The allocation rule is closely connected to the probability of arm allocation, $\mathbb{P}_{P_0}(A_t = a)$, or the ratio of arm allocation, $\frac{1}{T}\sum^T_{t=1}\mathbb{P}_{P_0}(A_t = a)$. Depending on the context, these probabilities and ratios have different implications for designing the allocation rule. However, for simplicity, we do not distinguish between them in our algorithm design. We refer to the arm allocation probability (or ratio) associated with optimal algorithms as the optimal allocation probability (or ratio).

For this problem, we derive a lower bound for the probability of misidentification $\mathbb{P}_{P_0}(\widehat{a}_T \neq a^\star(P_0))$ and develop an algorithm whose probability of misidentification aligns with the lower bound. 

\subsection{Background and Related Work}

We review the existing literature on fixed-budget BAI to clarify the issues. For well-designed algorithms $\pi$ and the true distribution $P_0$, the probability of misidentification, $\mathbb{P}_{P_0}\left( \widehat{a}^\pi_T \neq a^\star(P_0)\right)$, decreases exponentially as the budget $T$ approaches infinity. Specifically, for a value $C^\pi(P_0) > 0$ depending on $P_0$ and $\pi$ and independent of $T$, under a well-designed algorithm, we can have $\mathbb{P}_{P_0}\left( \widehat{a}^\pi_T \neq a^\star(P_0)\right) \approx \exp(-T C(P_0))$. A larger $C^\pi(P_0)$ indicates a faster decrease in the probability of misidentification. Thus, the goal is to develop algorithms that maximize $C^\pi(P_0)$.

To evaluate this exponential convergence, we use the following measure, called the \emph{complexity} of the probability of misidentification:
\begin{align}
\label{eq:complexity}
    -\frac{1}{T}\log\mathbb{P}_{P_0}\left( \widehat{a}^\pi_T \neq a^\star(P_0)\right).
\end{align}
Here, $C^\pi(P_0) \approx -\frac{1}{T}\log\mathbb{P}_{P_0}\left( \widehat{a}^\pi_T \neq a^\star(P_0)\right)$, and a larger complexity implies better performance. Typically, a lower bound represents the theoretical best performance (an upper bound on complexity), while an upper bound of an algorithm corresponds to the worst attainable performance (a lower bound on complexity). This complexity measure is widely used in studies on large-deviation evaluations in hypothesis testing \citep{Bahadur1960,Vaart1998}, ordinal optimization \citep{glynn2004large}, economics \citep{Kasy2021}, and BAI \citep{Kaufman2016complexity}.

Several studies evaluate performance using (simple) regret instead of the probability of misidentification. The regret is defined as
\[
\mathrm{Regret}_{P_0}(\pi) \coloneqq \mathbb{E}_{P_0}\left[\mu_{a^\star(P_0)}(P_0) - \mu_{a^\pi_T}(P_0)\right],
\]
where the expectation is taken over the randomness of $a^\pi_T$. This metric is referred to as the simple regret \citep{Bubeck2011} or policy regret \citep{Kasy2021}. The regret is closely related to complexity since the following bound holds:
\[
\mathrm{Regret}(\pi) \leq \max_{a\in[K]}\left(\mu_{a^\star(P_0)}(P_0)- \mu_a(P_0)\right)\mathbb{P}_{P_0}\left( \widehat{a}^\pi_T \neq a^\star(P_0)\right),
\]
which follows from the decomposition 
$
\mathrm{Regret}(\pi) = \sum_{a\in[K]}\left(\mu_{a^\star(P_0)}(P_0)- \mu_a(P_0)\right)\mathbb{P}_{P_0}\left( \widehat{a}^\pi_T = a\right)
$. 
Under a fixed $P_0$, the probability of misidentification and regret become asymptotically equivalent \citep{Kasy2021}, but they lead to different results in minimax and Bayesian analyses \citep{Bubeck2011,Komiyama2021}. In any analytical framework, the probability of misidentification plays a fundamental role. Since regret analysis introduces additional difficulties due to different assumptions and theoretical challenges, this study focuses solely on the probability of misidentification and its complexity to highlight our contributions explicitly. For an extension of our results to regret analysis, see \citet{kato2025minimaxoptimalsimpleregret}.

If the distributional information is fully known, optimal algorithms can be derived using large-deviation principles \citep{Gartner1977,Ellis1984}. For instance, when $Y_a$ follows a Gaussian distribution for all $a \in [K]$, \citet{chen2000} proposes an asymptotically optimal algorithm for Gaussian distributions. For general distributions, \citet{glynn2004large} and \citet{degenne2023existence} provide refined results. However, assuming complete knowledge of the distribution is unrealistic since it implies prior knowledge of the mean, variance, and the best arm.

In practice, BAI algorithms are often designed without assuming full knowledge of the underlying distributions \citep{Karnin2013,Shin2018}. When such knowledge is incomplete, the existence of optimal algorithms remains an open question \citep{kaufmann2020hdr,Ariu2021,Qin2022open,degenne2023existence}. Most studies in this area derive lower bounds for the complexity of the problem, as expressed in \eqref{eq:complexity}, and aim to design algorithms whose misidentification probability matches these bounds. 

\citet{Kaufman2016complexity} proposes a general framework for deriving information-theoretical lower bounds using change-of-measure arguments \citep{Lai1985}. These lower bounds depend on the underlying true distributions through the Kullback–Leibler (KL) divergence between the true and alternative distributions and are conjectured to be tight. When the number of arms is two ($K = 2$) and potential outcomes follow Gaussian distributions with known variances, \citet{Kaufman2016complexity} establishes the asymptotic optimality of the Neyman allocation by demonstrating that the probability of misidentification matches the derived lower bound. Since the variances are known, we can allocate arms according to the ratio of their standard deviations without the need for adaptive estimation during the experiment.

However, for cases involving multiple arms, more general distributions, or unknown variances, the derivation of lower bounds and the development of corresponding optimal algorithms remain open challenges. For example, \citet{Garivier2016} conjectures a lower bound for the multi-armed setting, but \citet{kaufmann2020hdr} identifies the \emph{reverse KL problem}, which suggests limitations of this conjecture.\footnote{To be more precise, the challenges arising from multiple arms, more general distributions, and unknown variances stem from different underlying issues. The reverse KL problem is primarily associated with the multi-armed case and persists even when distributions are Gaussian and variances are known. The issue related to general distributions arises when considering KL divergence. The problem of unknown variances becomes significant when evaluating the estimation error of the allocation ratio. While multiple complexities exist in this problem, one of the most notorious challenges is the reverse KL problem.}

The existence of optimal algorithms whose probability of misidentification matches these lower bounds, including constant terms, remains a longstanding issue. \citet{Kasy2021} attempts to design such an algorithm for lower bounds in \citet{Kaufman2016complexity},\footnote{Rigorously, \citet{Kasy2021} considers a large-deviation bound presented in \citet{glynn2004large}, which is equivalent to or closely related to the lower bounds proposed in \citet{Kaufman2016complexity}. For further details, see \citet{degenne2023existence}.} but \citet{Ariu2021} identifies technical issues in their proof related to the reverse KL problem in \citet{kaufmann2020hdr} and shows that there exists a distribution under which no algorithm achieves the lower bound conjectured in \citet{Kasy2021} and \citet{Kaufman2016complexity}. \citet{degenne2023existence} further refines these impossibility arguments, providing deeper insights into the challenges of achieving such optimality.

Motivated by these challenges, \citet{Komiyama2022} and \citet{Komiyama2021} propose minimax and Bayes-optimal algorithms, respectively. However, their algorithms exhibit constant gaps between their upper and lower bounds, leaving room for improvement in either the upper bound or the lower bound. Furthermore, \citet{Komiyama2022} does not account for the estimation error in the arm allocation probability (or ratio), which depends on parameters such as means and variances that are typically unknown and must be estimated during experiments. If the estimation error is included in their analysis, their algorithm fails to achieve the lower bound. \citet{degenne2023existence} refines the minimax approach. Note that minimax optimal algorithms have also been investigated by \citet{Carpentier2016} using a different approach; however, our minimax formulation is more closely related to that of \citet{Komiyama2022} and \citet{degenne2023existence}. 

Here, we elaborate on the issue related to the estimation error of the arm allocation probability. The allocation rule plays a critical role in best-arm identification. For two-armed problems with Gaussian outcomes, \citet{Kaufman2016complexity} shows that Neyman allocation, which allocates arms in proportion to the standard deviations of their outcomes, is asymptotically optimal when the variances are known. However, this result no longer holds when the variances are unknown. \citet{Jourdan2022} addresses the problem of variance estimation in the context of fixed-confidence BAI, which differs from our setting. Their analysis reveals that the algorithms and their theoretical properties change when variances must be estimated. 

Additionally, when considering multiple arms, the arm allocation probability depends on the mean and best arm, which are also unknown, not only on the variances. In the two-armed case, the unknown best arm does not pose a significant problem because the "optimal" arm allocation probability does not depend on which arm is the best.\footnote{Here, the "optimal" arm allocation probability refers to the allocation probability that, if followed, allows for identifying the best arm with a probability of misidentification that matches the lower bound.} However, in the multi-armed case, the optimal arm allocation probability may depend on the best arm, meaning that estimation error can influence the probability of misidentification.\footnote{Note that the optimal arm allocation probability has not been fully explored in existing studies. \citet{glynn2004large} conjectures the optimal arm allocation probability given full information about the distribution. Under the optimal arm allocation probability proposed by \citet{glynn2004large} and with additional restrictions on algorithms, \citet{degenne2023existence} show that the probability of misidentification aligns with the lower bound established by \citet{Kaufman2016complexity}. \citet{Shin2018} considers a heuristic algorithm for conducting BAI using the optimal arm allocation probability derived by \citet{glynn2004large} when the distribution information is unknown.}

In this study, we develop a novel algorithm and refine both the upper and lower bounds for multiple arms and general distributions while explicitly incorporating the estimation error in the arm allocation probability. The reverse KL problem is mitigated by adopting the minimax framework. Under the small-gap regime, we address the issues related to general distributions and the estimation error of the arm allocation probability. First, in the small-gap regime, the KL divergence can be approximated by that of Gaussian distributions, simplifying the analysis. Second, the estimation error of the arm allocation probability can be relatively ignored compared to the leading term of the performance metric, as the problem becomes increasingly difficult as the gaps approach zero.

Other studies addressing this open problem include \citet{Barrier2023}, \citet{atsidakou2023bayesian}, \citet{nguyen2024priordependent}, \citet{kato2024worstcase}, and \citet{wang2023uniformly}.

\begin{table}[t]
    \centering
        \caption{Summary of related work about the optimal algorithms.}
        \scalebox{0.8}{
    \begin{tabular}{|c|c|c|c|c|}
    \hline
        & \multicolumn{2}{|c|}{Two arms ($K = 2$)} & \multicolumn{2}{c|}{Multiple arms ($K \geq 3$)} \\
        \hline
        Arm allocation probability & Known & Unknown & Known & Unknown \\
        \hline
        \multirow{2}{*}{Optimal algorithm} & \multicolumn{2}{c|}{Neyman allocation} & \multicolumn{2}{c|}{Generalized Neyman allocation}\\
         & \citet{Kaufman2016complexity}  & \multicolumn{3}{c|}{This study} \\
         \hline
        Distribution & Gaussian & \multicolumn{3}{c|}{General} \\
        \hline
        Optimality & Globally optimal & \multicolumn{3}{|c|}{Locally optimal (small-gap regime)}\\
        \hline
    \end{tabular}
    }
    \label{tbl:related}
\end{table}

In fixed-confidence BAI, optimal algorithms have been proposed \citep{Garivier2016}. In Bayesian settings, asymptotically optimal algorithms have been developed based on posterior convergence rates \citep{Russo2016,Shang2020}, but these do not guarantee optimality for fixed-budget BAI \citep{Kasy2021,Ariu2021}.

\subsection{Contributions of This Study}
We propose an \emph{asymptotically minimax optimal algorithm} whose probability of misidentification matches the worst-case lower bound under the small-gap regime, where the gaps between the expected outcomes of the best and suboptimal arms are small. This property, which we term \emph{local (asymptotic) minimax optimality}, is demonstrated through the following contributions:
\begin{enumerate}[topsep=0pt, itemsep=0pt, partopsep=0pt, leftmargin=*]
    \item The derivation of the worst-case lower bound for fixed-budget BAI with multi-armed bandits (Theorem~\ref{thm:lower_bound}). 
    \item The proposal of the Generalized Neyman Allocation (GNA) algorithm (Algorithm~\ref{alg}).
    \item The derivation of the worst-case upper bound for the probability of misidentification of the GNA algorithm.
    \item A proof that the misidentification probability of the GNA algorithm aligns with the lower bound under the small-gap regime (Theorem~\ref{thm:upper_bound}).
    \item The algorithm's applicability to various distributions, including Bernoulli distributions, via Gaussian approximation.
\end{enumerate}

Informally, given an algorithm class $\Pi$ and a class of distributions restricted by the small-gap regime $\mathcal{P}^{\mathrm{small}}$, we have
\begin{align*}
   &\sup_{\pi \in \Pi} \inf_{P\in\mathcal{P}^{\mathrm{small}}}\limsup_{T \to \infty} -\frac{1}{\overline{\Delta}^2 T}\log \mathbb{P}_{P}\big( \widehat{a}^\pi_T \neq a^\star(P)\big)\\
   &\ \ \ \ \ \ \ \ \ \ \ \ \ \ \ \ \ \ \ \ \ \ \ \ \ \ \ \ \ \ \ \leq \sup_{\pi \in \Pi} \inf_{P\in\mathcal{P}^{\mathrm{small}}} C^{\pi*}(P) \leq \inf_{P\in\mathcal{P}^{\mathrm{small}}}\liminf_{T \to \infty}-\frac{1}{\underline{\Delta}^2T}\log\mathbb{P}_P\left(\widehat{a}^{\mathrm{GNA}}_T \neq a^\star(P)\right),
\end{align*}
where $C^{\pi*}(P)$ denotes an optimal constant, and $\widehat{a}^{\mathrm{GNA}}_T$ denotes the estimated best arm under the GNA algorithm.

In the subsequent sections, we define the algorithm class $\Pi$, the distribution class $\mathcal{P}^{\mathrm{small}}$, the optimal constant $C^\pi(P)$, and our proposed algorithm in greater detail. We summarize our contributions in Table~\ref{tbl:related}.

The complexity under the small-gap regime allows us to analyze performance while ignoring the estimation error of nuisance parameters, including the optimal arm allocation probability, which is not a parameter of interest. This regime reflects situations where, although the problem is challenging due to a small parameter of interest, the estimation error of nuisance parameters can be relatively neglected. Such a regime is also referred to as local Bahadur efficiency and has been widely used in the large-deviation analysis of statistical inference and decision-making with unknown parameters \citep{Kremer1979,Kremer1981,He1996}. 

It is important to note that this analysis differs from the local asymptotic normality framework, where parameters approach zero at the order of $\sqrt{T}$ for a sample size $T$. In the small-gap regime, the gaps approach zero independently of $T$. Thus, this small-gap analysis can be interpreted as a large-deviation analysis under distributions restricted to a specific range of parameters that are independent of $\sqrt{T}$. 

By employing this regime, we can circumvent the impossibility result shown in \citet{Ariu2021}. While \citet{Ariu2021} demonstrates that there exists a distribution under which no algorithm matches the lower bound, we address this by restricting the analysis to distributions with small gaps, effectively rejecting such problematic distributions.

For the expected outcome estimation, we employ the Adaptive Augmented Inverse Probability Weighting (A2IPW) estimator \citep{Kato2020adaptive,Kato2021adr,cook2023semiparametric}, also known as the doubly robust estimator \citep{BangRobins2005}, which is widely used across fields \citep{Laan2008TheCA}.

\section{Worst-case Lower Bound}
\label{sec:main_lower}
This section establishes a lower bound for the probability of misidentification, focusing on the worst-case scenario under the small-gap regime. Lower bounds not only reveal the theoretical performance limit but also provide valuable insights into the design of optimal algorithms.

\subsection{Distribution Class}
As a preparation, we define a class of distributions for $(Y_a)_{a\in[K]}$, which will be used to derive both lower and upper bounds. We refer to this class as a \emph{bandit model}. First, we define a class $\mathcal{P}_a$ of distributions for $Y_a$. Then, we define a bandit model as the set $(\mathcal{P}_a)_{a\in[K]}$ of $\mathcal{P}_a$.

\begin{definition}[Mean-parameterized distributions with finite variances]
\label{def:mean_param}
Let $\Theta \subset \mathbb{R}$ be a compact parameter space, and let $\mathcal{Y}$ be the support of $Y_a$ for all $a\in[K]$. Let $\sigma_a:\Theta \to (0, \infty)$ be a variance function that is continuous with respect to $\theta \in \Theta$. Let $P_{a, \mu_a}$ be a distribution of $Y_a$ parameterized by $\mu_a \in \Theta$. 
We define a class of distributions, $\mathcal{P}_a$ as
\begin{align}
    \mathcal{P}_a \coloneqq \mathcal{P}_a(\sigma_a(\cdot), \Theta, \mathcal{Y}) \coloneqq \Big\{ P_{a, \mu_a} \colon \mu_a\in\Theta,\ \eqref{enu:1},\ \eqref{enu:2},\ \mathrm{and}\ \eqref{enu:3} \Big\},
\end{align}
where \eqref{enu:1}, \eqref{enu:2}, and \eqref{enu:3} are defined as follows:
\begin{enumerate}[topsep=0pt, itemsep=0pt, partopsep=0pt, leftmargin=*]
\renewcommand\labelenumi{(\theenumi)}
    \item \label{enu:1} A distribution $P_{\mu_a, a}$ has a probability mass function or probability density function, denoted by $f_a(y\mid \mu_a)$. Additionally, $f_a(y\mid \mu_a) > 0$ holds for all $y \in \mathcal{Y}$ and $\mu_a \in \Theta$. 
    \item \label{enu:2} The variance of $Y_a$ under $P_{a, \mu_a}$ is $\sigma^2_a(\mu_a)$. For each $\mu_a \in \Theta$ and each $a\in[K]$, the Fisher information $I_a(\mu_a) > 0$ of $P_{a, \mu_a}$ exists and is equal to the inverse of the variance $1/\sigma^2_a(\mu_a)$.
    \item \label{enu:3} Let $\ell_a(\mu_a) = \ell_a(\mu_a\mid y) = \log f(y\mid \mu_a)$ be the likelihood function of $P_{\mu_a, a}$, and $\dot{\ell}_a$, $\ddot{\ell}_a$, and $\dddot{\ell}_a$ be the first, second, and third derivatives of $\ell_a$. The likelihood functions $\big\{\ell_a(\mu_a)\big\}_{a\in[K]}$ are three times differentiable and satisfy the following:
    \begin{enumerate}
        \item $\mathbb{E}_{P_{\mu_a, a}}\left[\dot{\ell}_a(\mu_a)\right] = 0$;
        \item $\mathbb{E}_{P_{\mu_a, a}}\left[\ddot{\ell}_a(\mu_a)\right] = -I_a(\mu_a) = 1/\sigma^2_{a}(\mu_a)$;
        \item For each $\mu_a \in \Theta$, there exist a neighborhood $U(\theta)$ and a function $u(y\mid \mu_a) \geq 0$, and the following holds:
        \begin{enumerate}
            \item $\left|\ddot{\ell}_a(\tau)\right| \leq u(y\mid \theta)\ \ \ \mathrm{for}\ U(\mu_a)$;
            \item $\mathbb{E}_{P_{\mu_a, a}}\left[u(Y\mid \mu_a)\right] < \infty$. 
        \end{enumerate}
    \end{enumerate}
\end{enumerate}
\end{definition}
This is a class of mean-parameterized distributions with finite variances. For example, Gaussian distributions with fixed variances and Bernoulli distributions belong to this class. In the case of Gaussian distributions, the variances are typically defined to be independent of $\mu_a$, meaning that for all $\mu_a$, $\sigma^2_a(\mu_a) = \sigma^2_a$, where $\sigma^2_a$ is a constant. In contrast, for Bernoulli distributions, the variances are given by $\sigma^2_a(\mu_a) = \mu_a(1-\mu_a)$.

For a distribution $P_{a, \mu_a} \in \mathcal{P}_a$ of $Y_a$, let $P_{\bm{\mu}} = (P_{a, \mu_a})_{a\in[K]}$ represent a set of distributions for $(Y_a)_{a\in[K]}$. Then, given $\bm{\sigma} = (\sigma_a(\cdot))_{a\in[K]}$, $\Theta$, $\mathcal{Y}$, and $0 < \underline{\Delta} < \overline{\Delta}$, we define the following class of distributions (bandit models) for $(Y_a)_{a\in[K]}$:
\[
\mathcal{P}\big(\underline{\Delta}, \overline{\Delta}\big) \coloneqq \mathcal{P}\big(\underline{\Delta}, \overline{\Delta}, \bm{\sigma}, \Theta, \mathcal{Y}\big) \coloneqq \Big\{ P_{\bm{\nu}} \colon \ \forall a \in [K],\ P_{a, \nu_a}\in\mathcal{P}_a,\ \ \forall b \in [K],\ \underline{\Delta} < \max_{a\in[K]}\nu_a - \nu_b \leq \overline{\Delta} \Big\},
\]
where $\underline{\Delta}$ and $\overline{\Delta}$ are the lower and upper bounds of the gap $\max_{a\in[K]}\nu_a - \nu_b$. Here, $P_{a, \mu_a}$ is the distribution of $Y_a$ marginalized over the other variables $(Y_b)_{b\neq a}$. We do not assume any specific relationships among $(Y_a)_{a\in[K]}$. Therefore, $\mathcal{P}\big(\underline{\Delta}, \overline{\Delta}\big)$ is simply a set of distributions and does not imply any joint distribution structure.

\subsection{Algorithm class}
Next, we define a class of algorithms and later show the asymptotic optimality for an algorithm belonging to this class. This study focuses on consistent algorithms that estimate the best arm with probability one as $T\to\infty$ \citep{Kaufman2016complexity}.

\begin{definition}[Consistent algorithms]\label{def:consistent}
We say that an algorithm $\pi$ is consistent if $\mathbb{P}_{P_0}(\widehat{a}^\pi_T = a^\star(P_0)) \to 1$ as $T\to \infty$ for any true distribution $P_0$ such that $a^\star(P_0)$ is unique. We denote the class of all possible consistent algorithms by $\Pi^{\mathrm{const}}$.
\end{definition}

Without this restriction, we could allow an algorithm that always returns arm $1$ independently of $P_0$. Such an algorithm would have zero probability of misidentification if $a^\star(P_0) = 1$ happens to hold. However, since such an algorithm is meaningless, we reject it by restricting the analysis to consistent algorithms.

\subsection{Worst-case Lower Bound}
Next, we derive a lower bound that any consistent algorithm must satisfy for mean-parameterized distributions. To obtain a tight lower bound, we employ the change-of-measure argument \citep{LeCam1972,LeCam1986,LehmCase98,Vaart1991,Vaart1998,Lai1985}. Using this argument, \citet{Kaufman2016complexity} develop a general framework for deriving lower bounds. In their results, the lower bounds are characterized by the KL divergence $\mathrm{KL}(P, Q)$ between two distributions $P$ and $Q$. Typically, we are interested in the KL divergence between the true distribution $P_0$ and an alternative hypothesis $Q$. 

These lower bounds not only establish the theoretical performance limit but also provide insights for designing optimal algorithms, particularly in the construction of allocation rules. For example, in fixed-confidence BAI, \citet{Garivier2016} develops an optimal allocation probability (ratio) of treatment arms based on the KL divergence. Assuming full knowledge of the distribution, \citet{glynn2004large} derives an optimal allocation probability (ratio) for fixed-budget BAI.

Thus, if the data-generating distribution $P_0$ were known, the KL divergence could be computed exactly, enabling the design of an asymptotically optimal algorithm based on the KL divergence, as demonstrated by \citet{glynn2004large}. However, in practice, $P_0$ is unknown—if it were, the best arm could be identified without experimentation. In the context of BAI, the challenge is to develop algorithms that do not rely on prior knowledge of $P_0$.\footnote{Such "oracle" algorithms are referred to as static proportion algorithms in \citet{degenne2023existence}.}

To account for the fact that $P_0$ is unknown in evaluation, we employ the statistical decision-making framework pioneered by \citet{Wald1945}. Among several evaluation criteria, such as Bayesian evaluation \citep{Komiyama2021}, we focus on the worst-case or minimax analysis. Specifically, we evaluate the worst-case probability of misidentification under the small-gap regime, defined as
\[
\limsup_{0 < \underline{\Delta} < \overline{\Delta} \to +0}\inf_{P\in\mathcal{P}\left(\underline{\Delta}, \overline{\Delta}\right)}\limsup_{T \to \infty} -\frac{1}{\overline{\Delta}^2T}\log \mathbb{P}_{P}\left(\widehat{a}^\pi_T \neq a^\star(P)\right).
\]
\citet{OTSU200853} introduces a similar minimax framework in the large-deviation analysis of statistical decision-making (treatment choice), in contrast to the local asymptotic normality analysis by \citet{Hirano2009}.

For the worst-case complexity, the following theorem provides a lower bound. The proof is presented in Appendix~\ref{appdx:lower_proof}.

\begin{theorem}[Worst-case Lower Bound]
\label{thm:lower_bound}
Fix $\bm{\sigma}$, $\Theta$, and $\mathcal{Y}$ in Definition~\ref{def:mean_param}. Given $\mathcal{P}\big(\underline{\Delta}, \overline{\Delta}\big) = \mathcal{P}\big(\underline{\Delta}, \overline{\Delta}, \bm{\sigma}, \Theta, \mathcal{Y}\big)$, 
any consistent algorithm $\pi\in\Pi^{\mathrm{const}}$ (Definition~\ref{def:consistent}) satisfies
\begin{align*}
    &\limsup_{0 < \underline{\Delta} < \overline{\Delta} \to +0}\inf_{P\in\mathcal{P}\big(\underline{\Delta}, \overline{\Delta}\big)}\limsup_{T \to \infty} -\frac{1}{\overline{\Delta}^2T}\log \mathbb{P}_{P}\left(\widehat{a}^\pi_T \neq a^\star(P)\right)\leq V^*,
\end{align*}
where 
\begin{align*}
    V^* &\coloneqq \min_{a\in[K], \mu \in \Theta}V(a, \mu)\\
    V(a, \mu) &\coloneqq \frac{1}{2\left(\sigma_{a}(\mu) + \sqrt{\sum_{b\in[K]\backslash\{a\}}\sigma^2_b(\mu)}\right)^2}.
\end{align*}
\end{theorem}

Our lower bound depends on the worst-case variance $(\sigma^2_a(\mu^\dagger))_{a\in[K]}$ with the worst-case mean parameter $\mu^\dagger \in \Theta$ defined as 
\[\min_{a\in[K]}\frac{1}{2\left(\sigma_{a}(\mu^\dagger) + \sqrt{\sum_{b\in[K]\backslash\{a\}}\sigma^2_b(\mu^\dagger)}\right)^2} = \min_{a\in[K], \mu \in \Theta}\frac{1}{2\left(\sigma_{a}(\mu) + \sqrt{\sum_{b\in[K]\backslash\{a\}}\sigma^2_b(\mu)}\right)^2}.\]

While several existing studies, such as \citet{Carpentier2016}, \citet{Komiyama2022}, \citet{yang2022minimax}, and \citet{degenne2023existence}, have introduced the minimax evaluation framework, there is a constant gap between the lower and upper bounds, and only the \emph{leading factors} in lower and upper bounds match. We conjecture that this is due to the estimation error of $P_0$ affecting the evaluation. To address this issue, we consider a restricted class for the distribution $P$, characterized by the small gap; that is, $\mu_{a^\star(P_0)} - \mu_a$ is sufficiently small for all $a\in[K]$. Under this regime, we can obtain matching lower and upper bounds, as shown in Section~\ref{sec:local_minimax}. We refer to this as local minimax optimality.

Here, we emphasize that in the two-armed case, the lower bound simplifies to 
\[V(a^\star(P_0), \mu_{a^\star(P_0)}(P_0)) = \frac{1}{2\left(\sigma_{1}(\mu) + \sigma_{2}(\mu)\right)^2},\] where the denominator is identical to the efficiency bound in ATE estimation when the propensity score is set to the ratio of the standard deviations \citep{Kato2020adaptive}.

\section{The GNA-A2IPW algorithm}
\label{sec:GNA_algorithm}
Based on the lower bounds, we design the GNA-A2IPW algorithm, which consists of the allocation rule using the generalized Neyman allocation and the estimation rule using the \emph{Adaptive Augmented Inverse Probability Weighting} (A2IPW) estimator. The pseudo-code is shown in Algorithm~\ref{alg}.  

\subsection{Allocation Rule: the Generalized Neyman Allocation}
\label{sec:sample}
First, we define a target allocation ratio. Let $\left(\sigma^2_a\right)_{a\in[K]}$ be the variances of $Y_a$ under the true distribution $P_0$. Using the variances, we define the oracle allocation rule as the one allocates arm $a$ following the probability $w^{\mathrm{GNA}}(a)$, defined as follows:
\begin{align*}
\bm{w}^{\mathrm{GNA}} \coloneqq \argmax_{\bm{w}\in\mathcal{W}}\min_{a\in[K]\backslash\{a^\star(P_0)\}}\frac{1}{\sigma^2_{a^\star(P_0)}/w_{a^\star(P_0)} + \sigma^2_a/w_a},
\end{align*}
where
\begin{align}
    \label{eq:opt_allocation}
    w^{\mathrm{GNA}}_{a^\star(P_0)} &= \frac{\sigma_{a^\star(P_0)}}{\sigma_{a^\star(P_0)} + \sqrt{\sum_{c\in[K]\backslash\{a^\star(P_0)\}}\sigma^2_c}},\\
    w^{\mathrm{GNA}}_a &= \frac{\sigma^2_a / \sqrt{\sum_{c\in[K]\backslash\{a^\star(P_0)\}}\sigma^2_c}}{\sigma_{a^\star(P_0)} + \sqrt{\sum_{c\in[K]\backslash\{a^\star(P_0)\}}\sigma^2_c}}=\frac{\left(1 - w^{\mathrm{GNA}}_{a^\star(P_0)}\right)\sigma^2_a}{\sum_{c\in[K]\backslash\{a^\star(P_0)\}}\sigma^2_c}\qquad \forall a \in [K]\backslash\{a^\star(P_0)\}.\nonumber
    \end{align}
Note that this probability is unknown since the best arm $a^\star(P_0)$ and the variance $(\sigma^2_a)_{a\in[K]}$ are unknown. This target allocation ratio is given from the proof of the lower bound in Theorem~\ref{thm:lower_bound}. We conjecture that an algorithm using the ratio \eqref{eq:opt_allocation} is optimal and aim to design an algorithm using it. We confirm that such an algorithm is actually optimal by checking that the performance matches the worst-case lower bound derived in the previous section; that is, the target allocation ratio is optimal.   
    
During our experiment, in each round $t$, we estimate $\bm{w}^{\mathrm{GNA}}$ using observed data until the round. Then, by using an estimated allocation ratio, we define our allocation rule. Our allocation rule is characterized by a sequence $\{\widehat{\bm{w}}^{\mathrm{GNA}}_t\}^T_{t=1}$, where $\widehat{\bm{w}}^{\mathrm{GNA}}_t = \left(\widehat{w}^{\mathrm{GNA}}_{a, t}\right)_{a\in[K]} \in \Delta^K$. The first $K$ rounds are the initialization phases. In each round $t = 1,\dots, K$, we set $\widehat{w}^{\mathrm{GNA}}_{1, t} = \cdots = \widehat{w}^{\mathrm{GNA}}_{K, t} = 1/K$ and sample $A_t = t$. Next, in each round $t = K + 1, K + 2, \dots, T$, we estimate $\sigma^2_a$ by using the past observations up to ($t-1$)-th round $\mathcal{F}_{t-1}$.
By using the estimate $\widehat{\sigma}^2_{a, t}$, we obtain $\widehat{\bm{w}}^{\mathrm{GNA}}_t$ as
\begin{align}
    \label{eq:sample_ave_weight}
    \widehat{w}^{\mathrm{GNA}}_{\widehat{a}_t,t} &= \frac{\widehat{\sigma}_{\widehat{a}_t, t}}{\widehat{\sigma}_{\widehat{a}_t, t} + \sqrt{\sum_{b\in[K]\backslash\{\widehat{a}_t\}}\widehat{\sigma}^2_{b, t}}},\\
    \widehat{w}^{\mathrm{GNA}}_{a,t} &= \frac{\widehat{\sigma}^2_{a, t} / \sqrt{\sum_{b\in[K]\backslash\{\widehat{a}_t\}}\widehat{\sigma}^2_{b, t}}}{\widehat{\sigma}_{\widehat{a}_t, t} + \sqrt{\sum_{b\in[K]\backslash\{\widehat{a}_t\}}\widehat{\sigma}^2_{b, t}}}=\left(1 - \widehat{w}^{\mathrm{GNA}}_{\widehat{a}_t}\right)\frac{\widehat{\sigma}^2_{a, t}}{\sum_{b\in[K]\backslash\{\widehat{a}_t\}}\widehat{\sigma}^2_{b, t}}\qquad \forall a \in [K]\backslash\{\widehat{a}_t\},\nonumber
    \end{align}
where $\widehat{a}_t = \argmax_{a\in[K]}\widetilde{\mu}_{a, t}$ and $\widetilde{\mu}_{a, t} = \frac{1}{\sum^{t-1}_{s=1}\mathbbm{1}\left[A_s = a\right]}\sum^{t-1}_{s=1}\mathbbm{1}\left[A_s = a\right]Y_{a, s}$. 
Then, we allocate arm $a$ with probability $\widehat{w}^{\mathrm{GNA}}_{a, t}$. 

In this study, we define an estimator $\widehat{\sigma}^2_{a,t}$ as 
\[\widehat{\sigma}^2_{a,t} \coloneqq \begin{cases}
    \widetilde{\sigma}^2_{a,t} &\mathrm{if}\ \ \widetilde{\sigma}^2_{a,t} \neq 0\\
    \eta &\mathrm{otherwise}
\end{cases},\]
 where $\eta > 0$ is a small positive value, $\widetilde{\sigma}^2_{a,t} \coloneqq \frac{1}{\sum^{t-1}_{s=1}\mathbbm{1}\left[A_s = a\right]}\sum^{t-1}_{s=1}\mathbbm{1}\left[A_s = a\right]\Big(Y_{a, s} - \widetilde{\mu}_{a, t}\Big)^2$. 

Note that $\widetilde{\sigma}^2_{a,t} > 0$ holds almost surely. However, to ensure a well-defined algorithm, we introduce $\eta$. The performance can be stabilized by modifying $\widehat{w}^{\mathrm{GNA}}_{a,t}$, provided that such modifications do not affect its consistency. For the modification ideas, see Section~3.3 in \citet{Kato2020adaptive}, which studies ATE estimation using the Neyman allocation in an adaptive experiment.

\subsection{Estimation Rule}
\label{sec:aipw_est}
At the end of the experiment (after observing $Y_T$), we estimate the best arm, an arm with the highest estimated expected outcome. We estimate the best arm by estimating the expected outcome $\mu_a$ for each $a\in[K]$. In this section, we define an estimation rule that employs an A2IPW estimator, proposed in \citet{Kato2020adaptive}. 

Let $\overline{C}_{\mu} > 0$ be a sufficiently large value. 
With a truncated version of the estimated expected reward $\widehat{\mu}_{a, t} \coloneqq \mathrm{thre}(\widetilde{\mu}_{a, t}, -\overline{C}_{\mu}, \overline{C}_{\mu})$, we define the A2IPW estimator of $\mu_a$ for each $a\in[K]$ as
\begin{align}
\label{eq:aipw}
&\widehat{\mu}^{\mathrm{A2IPW}}_{a, T} \coloneqq\frac{1}{T} \sum^T_{t=1}\left(\frac{\mathbbm{1}[A_t = a]\big(Y_{a, t}- \widehat{\mu}_{a, t}\big)}{\widehat{w}^{\mathrm{GNA}}_{a,t}} + \widehat{\mu}_{a, t}\right).
\end{align}
Here, $\overline{C}_{\mu}$ is introduced for technical purposes to bound the estimators and any large positive value can be used. 

At the end of the experiment (after the round $t=T$), we estimate $a^\star(P_0)$ as 
\begin{align}
\label{eq:recommend}
\widehat{a}^{\mathrm{GNA}}_T \coloneqq \argmax_{a\in[K]} \widehat{\mu}^{\mathrm{A2IPW}}_{a, T}
\end{align}

The A2IPW estimator consists of a martingale difference sequence (MDS), which enables the application of various asymptotic analysis tools and simplifies the theoretical analysis. Utilizing the MDS property, \citet{Kato2020adaptive} establish the asymptotic normality of the A2IPW estimator via the central limit theorem for an MDS. The unbiasedness of the A2IPW estimator follows directly from the definition of an MDS.

Although the A2IPW estimator is typically used in settings with covariates, we employ it in this study without covariates, as its MDS property significantly simplifies the theoretical analysis. While we conjecture that a naive sample mean estimator also possesses the same asymptotic theoretical properties, proving this result is not straightforward, as noted by \citet{hirano2003}.

\begin{algorithm}[tb]
   \caption{GNA algorithm}
   \label{alg}
\begin{algorithmic}
   \STATE {\bfseries Parameter:} Positive constants $\overline{C}_{\mu}$ and $\overline{C}_{\sigma^2}$.
   \STATE {\bfseries Initialization:} 
   \STATE For each $t = 1,2,\dots,K$, allocate $A_t=t$. For $a\in[K]$, set $\widehat{w}^{\mathrm{GNA}}_{a, t} = 1/K$.
   \FOR{$t=3$ to $T$}
   \STATE Estimate $\bm{w}^{\mathrm{GNA}}$ following \eqref{eq:sample_ave_weight}.
   \STATE Allocate $A_t = a$ with probability $\widehat{w}^{\mathrm{GNA}}_{a, t}$.
   \STATE Observe $Y_t = \sum_{a\in[K]}\mathbbm{1}[A_t = a]Y_{a, t}$. 
   \ENDFOR
   \STATE Construct $\widehat{\mu}^{\mathrm{A2IPW}}_{a, T}$ for $a\in[K]$. following \eqref{eq:aipw}.
   \STATE Estimate $a^\star(P_0)$ as $\widehat{a}_T$ following \eqref{eq:recommend}.
\end{algorithmic}
\end{algorithm}

\section{Worst-case Upper Bound and Local Asymptotic Minimax Optimality}
\label{sec:upper_bound}
This section provides an upper bound of the probability of misidentification of the GNA algorithm. 

\subsection{Worst-case Upper Bound}
We show the following upper bound for the probability of misidentification of the GNA algorithm, held for each $P_0$. We show the proof in Appendix~\ref{appdx:upper_proof}.
\begin{lemma}[Upper bound of the GNA algorithm]\label{lem:upper_bound}
Fix $\bm{\sigma}$, $\Theta$, and $\mathcal{Y}$ in Definition~\ref{def:mean_param}. Given $\mathcal{P}\big(\underline{\Delta}, \overline{\Delta}\big) = \mathcal{P}\big(\underline{\Delta}, \overline{\Delta}, \bm{\sigma}, \Theta, \mathcal{Y}\big)$, under the GNA algorithm, for any $\epsilon > 0$, there exist  $0 < \Delta < \Delta_0(\epsilon)$, the following holds: there exists $T_0(\Delta, \epsilon)$ such that for all $T > T_0(\Delta, \epsilon)$, it holds that 
\begin{align*}
     & - \frac{1}{T}\log \mathbb{P}_{P_0}
     \left(\widehat{a}^{\mathrm{GNA}}_T \neq a^\star(P_0)\right) \geq - \underline{\Delta}^2V(a^\star(P_0), \mu_{a^\star(P_0)}(P_0)) - \epsilon \underline{\Delta}^2.
\end{align*}
for all $P_0 \in \mathcal{P}\big(\underline{\Delta}, \overline{\Delta}\big)$ such that $\underline{\Delta} < \overline{\Delta} < \Delta$. 
\end{lemma}
Here, recall that 
$V(a, \mu) = \frac{1}{2\left(\sigma_{a}(\mu) + \sqrt{\sum_{b\in[K]\backslash\{a\}}\sigma^2_b(\mu)}\right)^2}$. 

Then, by considering the worst-case scenario for $P_0$, we obtain the following worst-case upper bound.
\begin{theorem}[Worst-case upper bound of the GNA algorithm]\label{thm:upper_bound}
Fix $\bm{\sigma}$, $\Theta$, and $\mathcal{Y}$ in Definition~\ref{def:mean_param}. Given $\mathcal{P}\big(\underline{\Delta}, \overline{\Delta}\big) = \mathcal{P}\big(\underline{\Delta}, \overline{\Delta}, \bm{\sigma}, \Theta, \mathcal{Y}\big)$, the GNA algorithm satisfies
\begin{align*}
   &\liminf_{0 < \underline{\Delta} < \overline{\Delta} \to +0} \inf_{P \in \mathcal{P}\big(\underline{\Delta}, \overline{\Delta}\big)} \liminf_{T \to \infty}-\frac{1}{\underline{\Delta}^2T}\log\mathbb{P}_P\left(\widehat{a}^{\mathrm{GNA}}_T \neq a^\star(P)\right) \geq V^* = \min_{a\in [K], \mu \in \Theta} V(a, \mu).
\end{align*}
\end{theorem}
Recall that $\inf_{P \in \mathcal{P}\big(\underline{\Delta}, \overline{\Delta}\big)}$ represents the worst case, as a smaller value of 
$-\frac{1}{T}\log\mathbb{P}_P\left(\widehat{a}^{\mathrm{GNA}}_T \neq a^\star(P)\right)$
implies a lower probability of misidentification, i.e., $\mathbb{P}_P\left(\widehat{a}^{\mathrm{GNA}}_T \neq a^\star(P)\right)$ becomes smaller.

\subsection{Local Minimax Optimality}
\label{sec:local_minimax}
This section proves the local minimax optimality of our proposed GNA algorithm. The following result demonstrates that the probability of misidentification of the GNA algorithm matches the worst-case lower bound derived under the small-gap regime. This establishes that our proposed algorithm is \emph{local asymptotic minimax optimal}.

\begin{theorem}[Local asymptotic minimax optimality]
Fix $\bm{\sigma}$, $\Theta$, and $\mathcal{Y}$ in Definition~\ref{def:mean_param}. Given $\mathcal{P}\big(\underline{\Delta}, \overline{\Delta}\big) = \mathcal{P}\big(\underline{\Delta}, \overline{\Delta}, \bm{\sigma}, \Theta, \mathcal{Y}\big)$, for each $P_0 \in \mathcal{P}\big(\underline{\Delta}, \overline{\Delta}\big)$, the GNA algorithm satisfies
\begin{align*}
   &\sup_{\pi \in \Pi^{\mathrm{const}}} \limsup_{0 < \underline{\Delta} < \overline{\Delta} \to +0}\inf_{P\in\mathcal{P}\left(\underline{\Delta}, \overline{\Delta}\right)}\limsup_{T \to \infty} -\frac{1}{\overline{\Delta}^2 T}\log \mathbb{P}_{P}\big( \widehat{a}^\pi_T \neq a^\star(P)\big)\\
   &\ \ \ \ \ \ \ \ \ \ \ \ \ \ \ \ \ \ \ \ \ \ \ \ \ \ \ \ \ \ \ \ \ \leq V^* \leq \liminf_{0 < \underline{\Delta} < \overline{\Delta} \to +0}\inf_{P \in \mathcal{P}\big(\underline{\Delta}, \overline{\Delta}\big)}\liminf_{T \to \infty}-\frac{1}{\underline{\Delta}^2T}\log\mathbb{P}_P\left(\widehat{a}^{\mathrm{GNA}}_T \neq a^\star(P)\right).
\end{align*}
\end{theorem}
Note again that the upper and lower bounds appear flipped due to the properties of the logarithm function, $\log(x)$.

\begin{figure}[t]
  \centering
\includegraphics[width=0.7\linewidth]{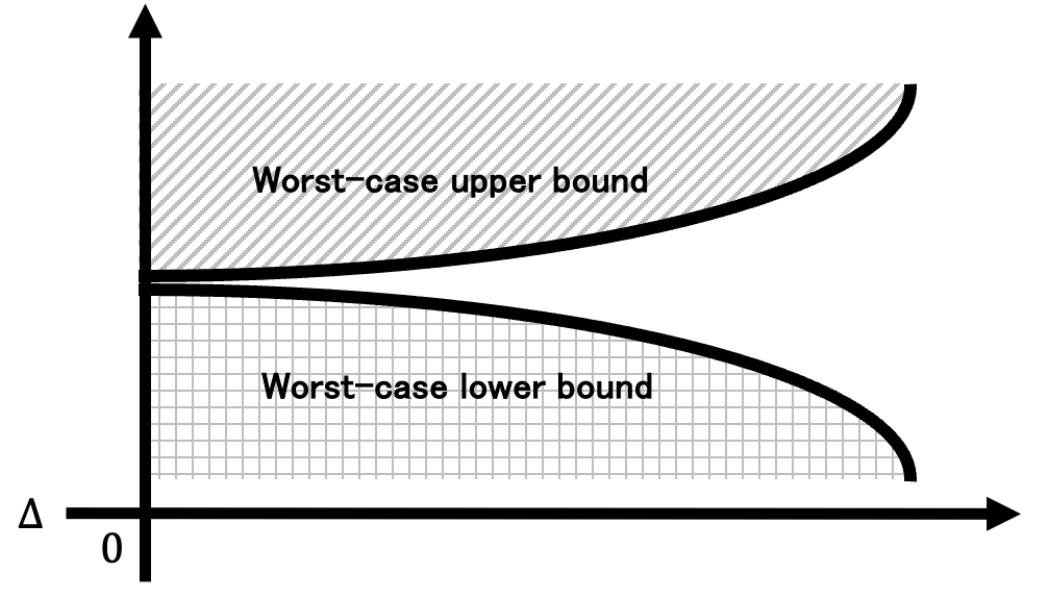}
\caption{Illustration of local asymptotic minimax optimality. The $y$-axis represents the probability of misidentification, while the $x$-axis represents the gap $\Delta = \overline{\Delta} = \underline{\Delta}$ (for simplicity, we set $\Delta = \overline{\Delta} = \underline{\Delta}$). The upper (red) region represents the upper bound, and the lower (blue) region represents the lower bound, which converges as $\Delta \to 0$ (small-gap regime).}
\label{fig:concept:smallgap}
\end{figure}

We illustrate the concept of local asymptotic minimax optimality in Figure~\ref{fig:concept:smallgap}. In the figure, the upper (red) region represents the upper bound, while the lower (blue) region represents the lower bound, and they converge as $\Delta \to 0$ (small-gap). That is, although a discrepancy remains between the lower and upper bounds, it diminishes as the gap approaches zero.

\subsection{Intuitive Explanation}
We provide an intuitive explanation of why the GNA algorithm is effective. BAI algorithms are closely connected to constructing tight confidence intervals for testing whether $\mu_{a^\star(P_0)}(P_0) \neq \mu_a(P_0)$ for all $a\neq a^\star(P_0)$. To accurately identify the best arm $a^\star(P_0)$, given $a^\star(P_0)$, we aim to allocate treatment arms such that the discrepancies between the confidence intervals of the expected outcomes of the best and suboptimal arms become large.

First, we consider how the GNA algorithm operates when $K = 2$, which corresponds to the standard Neyman allocation algorithm. For instance, assuming that the best arm is $a^\star(P_0) = 1$, we estimate the ATE $\mu_1(P_0) - \mu_2(P_0)$ to test whether $a^\star(P_0) \neq 1$. If the null hypothesis is rejected, we conclude that the alternative hypothesis $a^\star(P_0) = 1$ is accepted. To efficiently test this hypothesis using asymptotic properties, it is crucial to minimize the asymptotic variance of the ATE estimators. In ATE estimation, the asymptotic variance of an efficient ATE estimator is given by $
\frac{\sigma^2_{1}(\mu_1(P_0))}{w_1} + \frac{\sigma^2_{2}(\mu_2(P_0))}{w_2}$, 
where $(w_1, w_2)$ represents the probability of treatment assignment [propensity score,][]\citep{hahn1998role}. This variance is minimized when applying the Neyman allocation for $(w_1, w_2)$. Notably, the Neyman allocation does not depend on which arm is the best, ensuring its effectiveness when $K = 2$.

Next, we consider the case where $K \geq 3$. In this scenario, assuming that the best arm is $a^\star(P_0) = a^\dagger \in [K]$, we test whether $a^\star(P_0) \neq a^\dagger$. If the null hypothesis is rejected, we accept the alternative hypothesis $a^\star(P_0) = a^\dagger$. To efficiently test this hypothesis using asymptotic properties, we aim to estimate the ATEs $\mu_{a^\dagger}(P_0) - \mu_a(P_0)$ for all $a\neq a^\dagger$ as efficiently as possible. A challenge arises because focusing solely on one arm $b\neq a^\dagger$ to estimate the ATE $\mu_{a^\dagger}(P_0) - \mu_b(P_0)$ efficiently may result in inefficient estimation of other ATEs $\mu_{a^\dagger}(P_0) - \mu_c(P_0)$ for $c\neq a^\dagger, b$. Therefore, we set the propensity score as 
\[
\bm{w}^* \coloneqq \argmin_{\bm{w}\in\mathcal{W}}\max_{a\neq a^\dagger} \left\{ \frac{\sigma^2_{a^\dagger}(\mu_{a^\dagger}(P_0))}{w_{a^\dagger}} + \frac{\sigma^2_{a}(\mu_{a^\dagger}(P_0))}{w_a} \right\}.
\]
The minimizer $\bm{w}^*$ in this case is given by $\bm{w}^{\mathrm{GNA}}$ in the GNA algorithm. Thus, by setting the arm allocation probability to $\bm{w}^{\mathrm{GNA}}$, we can efficiently test whether $a^\star(P_0) \neq a^\dagger$.

\section{BAI with Bernoulli Bandits}
This section examines the behavior of the GNA strategy when potential outcomes follow Bernoulli distributions, representing a specific case of the general results. Specifically, we consider the following model:
\begin{align*}
    \mathcal{P}^{\mathrm{B}}_a \coloneqq \Big\{\mathrm{Bernoulli}(\mu_a) \colon \mu_a \in \Theta\Big\},
\end{align*}
which is a subset of $\mathcal{P}_a$, with $\sigma_a(\mu) = \mu(1 - \mu)$ representing the variance of the Bernoulli outcomes.

In the case of two-armed Bernoulli bandits, the GNA strategy allocates arms uniformly, with the allocation ratio $w^{\mathrm{GNA}}_1 = w^{\mathrm{GNA}}_2 = \frac{1}{2}$. This uniform allocation occurs because the variances are identical under the small-gap regime. This result is consistent with the findings of \citet{Kaufman2016complexity} and \citet{wang2023uniformly}, which state that such a uniform allocation is optimal.

For multi-armed Bernoulli bandits with $K \geq 2$, the allocation ratio becomes
\begin{align*}
    w^{\mathrm{GNA}}_{a^\star(P_0)} &= \frac{1}{1 + \sqrt{K-1}}, \\
    w^{\mathrm{GNA}}_{a} &= \frac{1}{\sqrt{K-1}(1 + \sqrt{K-1})} = \frac{1}{K - 1 + \sqrt{K-1}} \quad \forall a \in [K] \setminus \{a^\star(P_0)\}.
\end{align*}
This allocation ratio depends on the best arm $a^\star(P_0)$, which is needed to be estimated.

The lower and upper bounds are given by
\begin{align*}
    V^* = \frac{1}{2\left(0.5 + \sqrt{(K - 1) \cdot 0.5}\right)^2}.
\end{align*}

In Bernoulli bandits, we do not have to estimate the variances since, under the small-gap regime with the worst-case analysis, they are determined from the worst-case mean parameters. Additionally, the variances approach the same values as the gaps approach zero. Therefore, there is no need to estimate the variances. However, note that when $K \geq 3$, we need to estimate the best arm during the experiment. 

Our results also address an open issue highlighted in \citet{Kasy2021corrigendum}, which provides a correction to \citet{Kasy2021}. The technical issue in \citet{Kasy2021} arises in their proof due to flipping $\mathrm{KL}(Q_a, P_a)$ and $\mathrm{KL}(P_a, Q_a)$, where $P_a$ is the true distribution and $Q_a$ is the alternative distribution. In their correction, they say that ``An interesting question for future work will be to bound the differences (between $\mathrm{KL}(Q_a, P_a)$ and $\mathrm{KL}(P_a, Q_a)$) in the implied optimal allocations; in many settings, these will be very small.''

Our small-gap regime provides a case where these differences are indeed small. In this regime, while \citet{Kasy2021} proposes exploration sampling for BAI with Bernoulli outcomes using a Bayesian method to compute the allocation rule, we demonstrate that this method can be significantly simplified. As shown above, the allocation rule can be computed in closed form.


\begin{table}[t]
\caption{The results with $\mu_1 = 1.00$, $\mu_2 = 0.90$, $\mu_a \sim \mathrm{Uniform}[0.90, 0.95]$ for all $a\in[K]\backslash\{1, 2\}$, and $\overline{\sigma} = 3$ for $K = 3$ (Upper table) and $K = 5$ (Lower table). We report the empirical probability of misidentification (\%) at $T \in \{5000, 10000, 20000, 30000, 40000, 50000\}$.}
    \centering
\scalebox{0.8}{
\begin{tabular}{lllllll}
\toprule
T & 5000 & 10000 & 20000 & 30000 & 40000 & 50000 \\
\midrule
GNA (\%) & 0.60 \% & 1.20 \% & 2.20 \% & 0.20 \% & 3.80 \% & 2.60 \% \\
GJ (Oracle) (\%) & 2.60 \% & 0.80 \% & 9.00 \% & 2.40 \% & 5.60 \% & 1.60 \% \\
Uniform (\%) & 2.40 \% & 1.20 \% & 4.20 \% & 5.60 \% & 5.40 \% & 1.20 \% \\
SH (\%) & 2.60 \% & 0.20 \% & 5.00 \% & 0.60 \% & 8.60 \% & 1.00 \% \\
\bottomrule
\end{tabular}
}
\vspace{5mm}
\scalebox{0.8}{
\begin{tabular}{lllllll}
\toprule
T & 5000 & 10000 & 20000 & 30000 & 40000 & 50000 \\
\midrule
GNA (\%) & 2.60 \% & 14.00 \% & 2.60 \% & 2.00 \% & 14.20 \% & 7.00 \% \\
GJ (Oracle) (\%) & 5.40 \% & 11.80 \% & 19.20 \% & 11.00 \% & 10.20 \% & 8.00 \% \\
Uniform (\%) & 7.60 \% & 17.80 \% & 35.60 \% & 9.00 \% & 10.60 \% & 5.80 \% \\
SH (\%) & 8.00 \% & 18.00 \% & 13.40 \% & 13.60 \% & 7.60 \% & 15.60 \% \\
\bottomrule
\end{tabular}
}
\label{tab:res1}
\caption{The results with $\mu_1 = 1.00$, $\mu_2 = 0.90$, and $\mu_a \sim \mathrm{Uniform}[0.90, 0.95]$ for all $a\in[K]\backslash\{1, 2\}$, with $\overline{\sigma} = 5$ for $K = 3$ (Upper table) and $K = 5$ (Lower table). We report the empirical probability of misidentification (\%) at $T \in \{5000, 10000, 20000, 30000, 40000, 50000\}$.}
    \centering
\scalebox{0.8}{
\begin{tabular}{lllllll}
\toprule
T & 5000 & 10000 & 20000 & 30000 & 40000 & 50000 \\
\midrule
GNA (\%) & 3.20 \% & 6.00 \% & 1.80 \% & 1.20 \% & 6.40 \% & 60.40 \% \\
GJ (Oracle) (\%) & 3.60 \% & 1.40 \% & 19.00 \% & 2.80 \% & 9.80 \% & 41.20 \% \\
Uniform (\%) & 2.80 \% & 1.40 \% & 25.20 \% & 7.20 \% & 4.40 \% & 54.60 \% \\
SH (\%) & 2.60 \% & 0.80 \% & 16.40 \% & 3.80 \% & 5.20 \% & 16.40 \% \\
\bottomrule
\end{tabular}
}
\scalebox{0.8}{
\begin{tabular}{lllllll}
\toprule
T & 5000 & 10000 & 20000 & 30000 & 40000 & 50000 \\
\midrule
GNA (\%) & 2.60 \% & 30.20 \% & 4.00 \% & 3.60 \% & 14.60 \% & 9.60 \% \\
GJ (Oracle) (\%) & 9.20 \% & 14.60 \% & 29.00 \% & 14.20 \% & 13.20 \% & 5.80 \% \\
Uniform (\%) & 8.60 \% & 20.60 \% & 50.80 \% & 11.00 \% & 12.00 \% & 5.20 \% \\
SH (\%) & 11.00 \% & 19.80 \% & 20.00 \% & 16.20 \% & 8.20 \% & 7.20 \% \\
\bottomrule
\end{tabular} 
}
\label{tab:res2}
\end{table}

\begin{table}[t]
\caption{The results with $\mu_1 = 1.00$ $\mu_a = 0.95$ for all $a\in [K]\backslash\{1\}$, and $\overline{\sigma} = 3$ for $K = 3$ (Upper table) and $K = 5$ (Lower table). We report the empirical probability of misidentification (\%) at $T \in \{5000, 10000, 20000, 30000, 40000, 50000\}$.}
    \centering
\scalebox{0.8}{
\begin{tabular}{lllllll}
\toprule
T & 5000 & 10000 & 20000 & 30000 & 40000 & 50000 \\
\midrule
GNA (\%) & 1.60 \% & 1.20 \% & 3.00 \% & 0.70 \% & 0.10 \% & 1.90 \% \\
GJ (Oracle) (\%) & 3.80 \% & 0.90 \% & 4.40 \% & 2.30 \% & 0.20 \% & 5.20 \% \\
Uniform (\%) & 1.70 \% & 2.80 \% & 6.00 \% & 8.20 \% & 1.10 \% & 6.10 \% \\
SH (\%) & 2.50 \% & 1.30 \% & 8.20 \% & 2.80 \% & 0.90 \% & 1.40 \% \\
\bottomrule
\end{tabular}
}
\scalebox{0.8}{
\begin{tabular}{lllllll}
\toprule
T & 5000 & 10000 & 20000 & 30000 & 40000 & 50000 \\
\midrule
GNA (\%) & 6.20 \% & 0.20 \% & 13.50 \% & 8.10 \% & 1.50 \% & 0.00 \% \\
GJ (Oracle) (\%) & 4.50 \% & 4.70 \% & 8.50 \% & 6.60 \% & 2.40 \% & 4.30 \% \\
Uniform (\%) & 2.60 \% & 8.10 \% & 6.90 \% & 7.00 \% & 1.40 \% & 3.20 \% \\
SH (\%) & 3.60 \% & 3.30 \% & 15.50 \% & 7.50 \% & 2.10 \% & 5.80 \% \\
\bottomrule
\end{tabular}
}
\label{tab:res3}
\caption{The results with $\mu_1 = 1.00$, $\mu_a = 0.95$ for all $a\in [K]\backslash\{1\}$, and $\overline{\sigma} = 5$ for $K = 3$ (Upper table) and $K = 5$ (Lower table). We report the empirical probability of misidentification (\%) at $T \in \{5000, 10000, 20000, 30000, 40000, 50000\}$.}
    \centering
\scalebox{0.8}{
\begin{tabular}{lllllll}
\toprule
T & 5000 & 10000 & 20000 & 30000 & 40000 & 50000 \\
\midrule
GNA (\%) & 1.60 \% & 0.80 \% & 2.80 \% & 2.60 \% & 0.50 \% & 7.30 \% \\
GJ (Oracle) (\%) & 8.10 \% & 4.30 \% & 12.20 \% & 3.20 \% & 0.40 \% & 6.10 \% \\
Uniform (\%) & 2.50 \% & 8.10 \% & 14.00 \% & 11.40 \% & 1.90 \% & 7.60 \% \\
SH (\%) & 3.50 \% & 2.70 \% & 17.50 \% & 3.40 \% & 1.80 \% & 6.10 \% \\
\bottomrule
\end{tabular}
}
\scalebox{0.8}{
\begin{tabular}{lllllll}
\toprule
T & 5000 & 10000 & 20000 & 30000 & 40000 & 50000 \\
\midrule
GNA (\%) & 8.60 \% & 2.00 \% & 21.20 \% & 7.60 \% & 4.00 \% & 0.00 \% \\
GJ (Oracle) (\%) & 14.40 \% & 9.10 \% & 12.00 \% & 9.40 \% & 5.10 \% & 6.90 \% \\
Uniform (\%) & 12.60 \% & 10.70 \% & 9.30 \% & 10.60 \% & 2.10 \% & 3.60 \% \\
SH (\%) & 5.80 \% & 9.40 \% & 20.40 \% & 10.10 \% & 3.20 \% & 7.30 \% \\
\bottomrule
\end{tabular}
}
\label{tab:res4}
\end{table}

\begin{figure}[t!]
  \centering
  \includegraphics[width=0.9\linewidth]{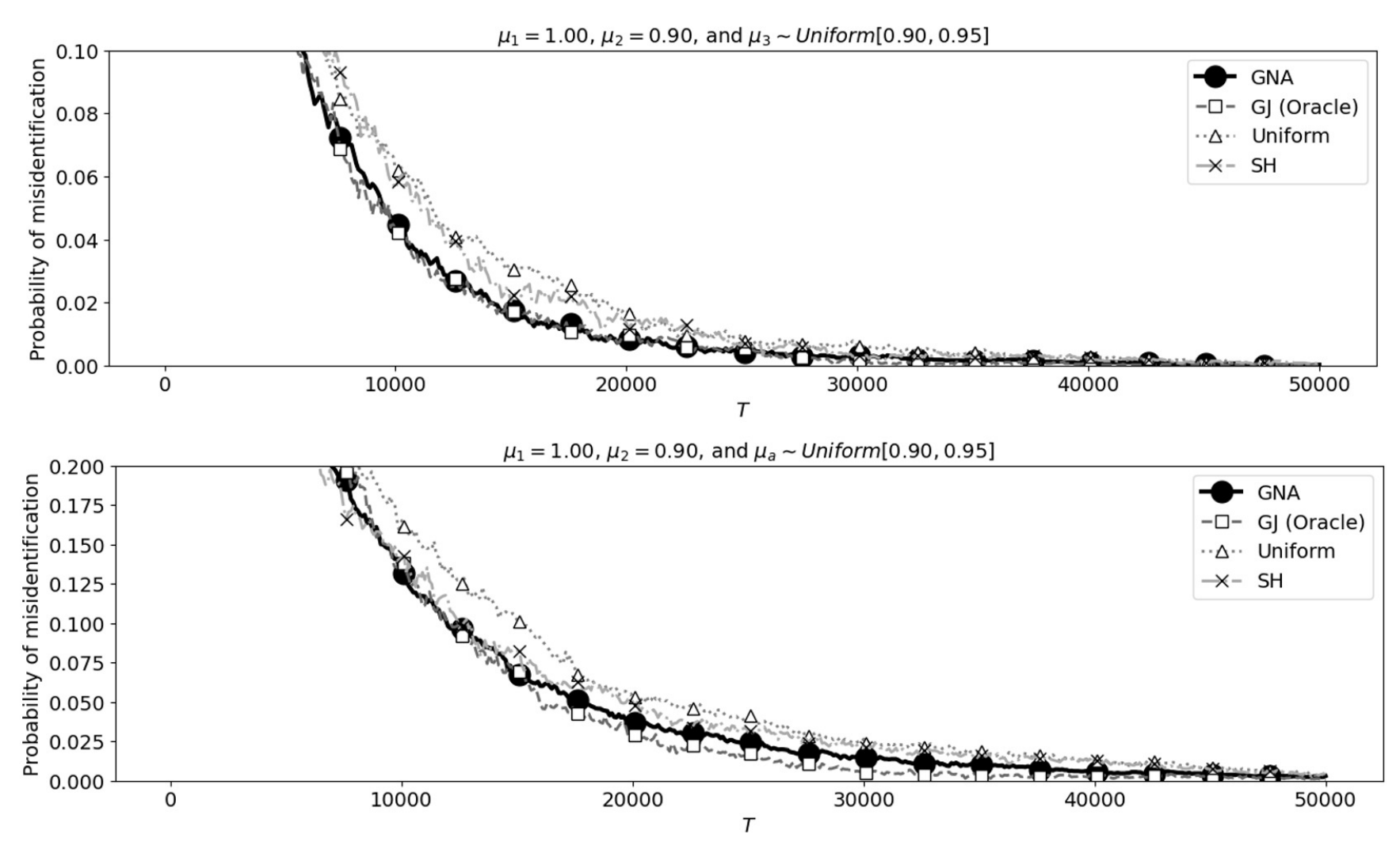}
  \vspace{-3mm}
  \caption{The results with $\mu_1 = 1.00$, $\mu_2 = 0.90$, $\mu_a \sim \mathrm{Uniform}[0.90, 0.95]$ for all $a\in[K]\backslash\{1, 2\}$, and $\overline{\sigma} = 3$ for $K = 3$ (Upper graph) and $K = 5$ (Lower graph). We report the empirical probability of misidentification at $T \in \{100, 200, 300, \dots, 49900, 50000\}$.}
      \label{fig:res1}
    \vspace{3mm}
  \includegraphics[width=0.9\linewidth]{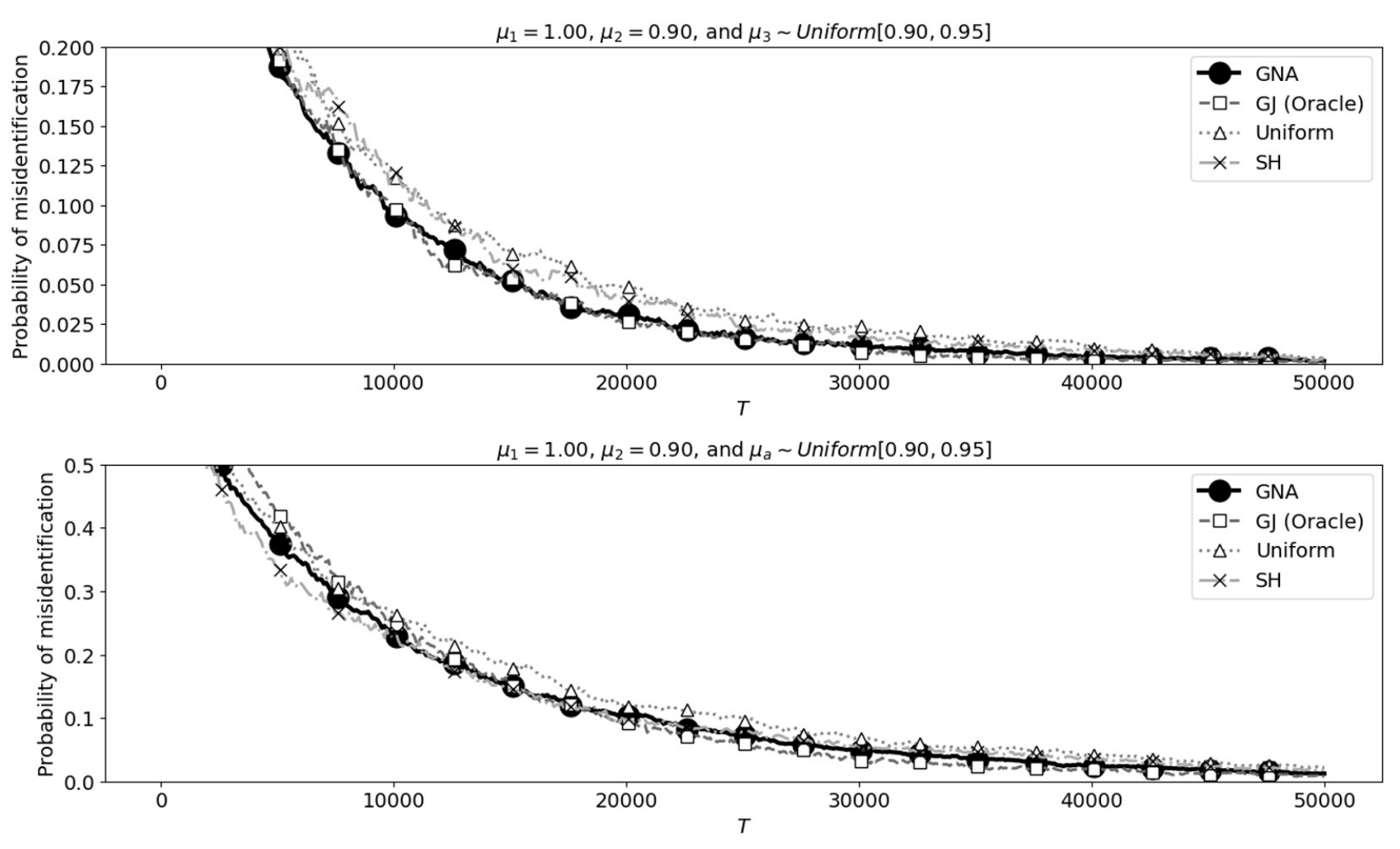}
  \vspace{-3mm}
  \caption{The results with $\mu_1 = 1.00$, $\mu_2 = 0.90$, and $\mu_a \sim \mathrm{Uniform}[0.90, 0.95]$ for all $a\in[K]\backslash\{1, 2\}$, with $\overline{\sigma} = 5$, for $K = 3$ (Upper graph) and $K = 5$ (Lower graph). We report the empirical probability of misidentification at $T \in \{100, 200, 300, \dots, 49900, 50000\}$.}
      \label{fig:res2}
\vspace{-5mm}
\end{figure}

\begin{figure}[t!]
  \centering
  \includegraphics[width=0.9\linewidth]{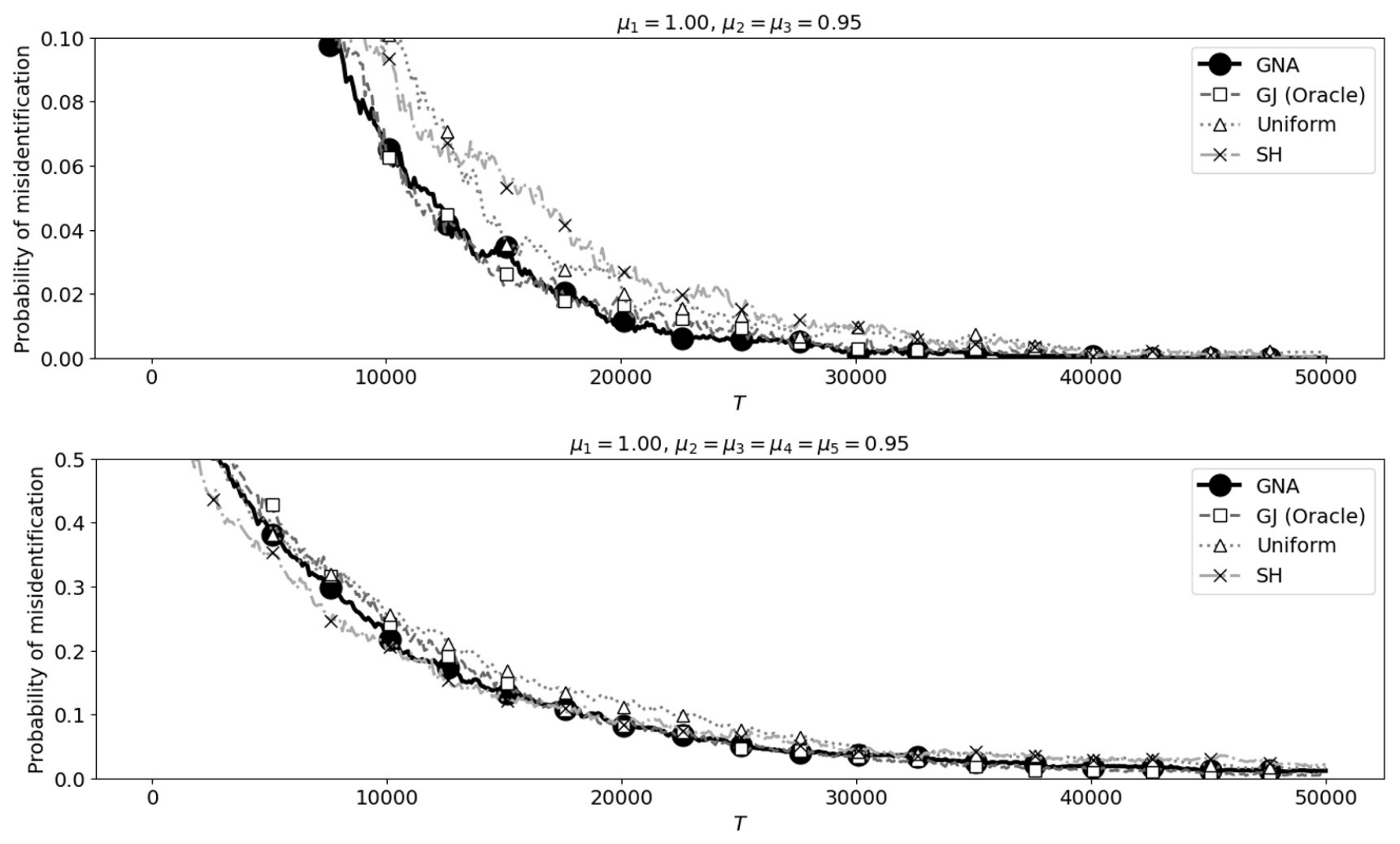}
  \vspace{-3mm}
  \caption{The results with $\mu_1 = 1.00$ $\mu_a = 0.95$ for all $a\in [K]\backslash\{1\}$, and $\overline{\sigma} = 3$ for $K = 3$ (Upper graph) and $K = 5$ (Lower graph). We report the empirical probability of misidentification at $T \in \{100, 200, 300, \dots, 49900, 50000\}$.}
      \label{fig:res3}
    \vspace{3mm}
  \includegraphics[width=0.9\linewidth]{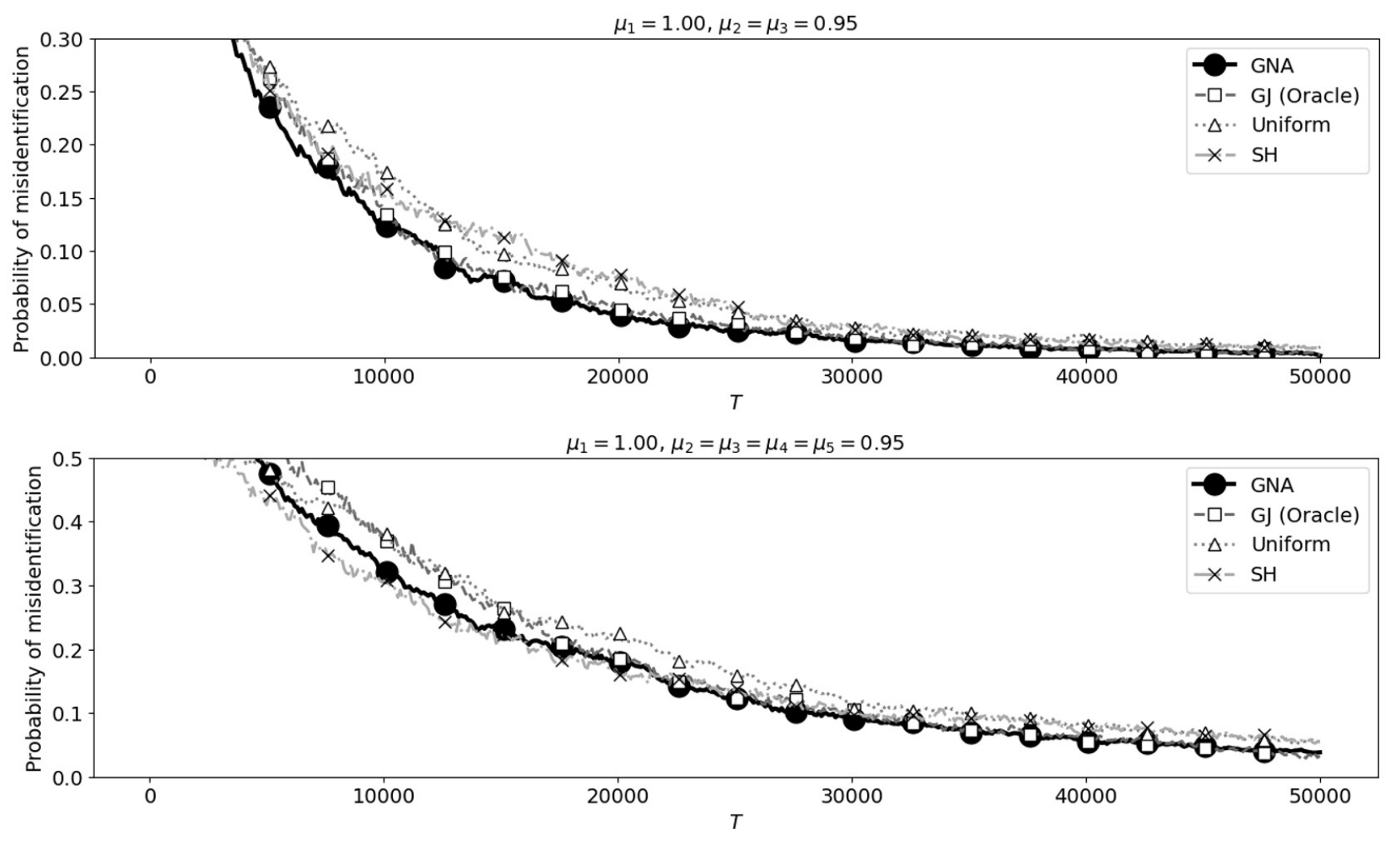}
  \vspace{-3mm}
  \caption{The results with $\mu_1 = 1.00$, $\mu_a = 0.95$ for all $a\in [K]\backslash\{1\}$, and $\overline{\sigma} = 5$ for $K = 3$ (Upper graph) and $K = 5$ (Lower graph). We report the empirical probability of misidentification at $T \in \{100, 200, 300, \dots, 49900, 50000\}$.}
      \label{fig:res4}
\vspace{-5mm}
\end{figure}





\section{Simulation Studies}
\label{sec:exp_results}
In this section, we investigate the empirical performance of our proposed GNA algorithm. We compare GNA with the Uniform-EBan algorithm \citep[Uniform,][]{Bubeck2011}, which allocates arms with the GNA given known variances (GNA); an equal allocation ratio ($1/K$); the successive rejects algorithm \citep[SR,][]{Audibert2010}; and the large-deviation optimal (GJ) algorithm proposed by \citet{glynn2004large}

The GJ algorithm is an oracle algorithm that is proven to be asymptotically optimal in cases where we have full knowledge of the distributions, including mean parameters and the identity of the best arm. Note that this algorithm is practically infeasible. The GNA algorithm does not estimate variances but directly allocates arms according to the ratio $\bm{w}^{\mathrm{GNA}}$ using known variances.

Let $K \in \{3, 5\}$. The best arm is arm $1$ with $\mu_1 = \mu_1(P_0) = 1$. We consider two cases. In the first case, we set $\mu_a = \mu_a(P_0) = 0.95$ for all $a\in[K]\backslash\{1\}$. In the second case, we set $\mu_2 = \mu_2(P_0) = 0.95$ and choose $\mu_a = \mu_a(P_0)$ from a uniform distribution with support $[0.90, 0.95]$ for all $a\in[K]\backslash\{1, 2\}$.

The variances are given as a permutation of the set $\{\overline{\sigma}, \underline{\sigma}, \sigma_{(3)},\dots,\sigma_{(K)}\}$, where $\overline{\sigma}$ is chosen from $\{5, 10\}$, $\underline{\sigma} = 0.1$, and $\sigma_{(3)}$ is chosen from a uniform distribution with support $[\underline{\sigma}, \overline{\sigma}]$.

We investigate the performance for each $T\in\{100, 200, 300, \cdots, 49900, 50000\}$.
We conduct $3000$ independent trials for each setting and compute the empirical probability of misidentification.
We plot the empirical probability of misidentification $\widehat{p}$ in Tables~\ref{tab:res1}--\ref{tab:res4} and Figures~\ref{fig:res1}--~\ref{fig:res4}. Note that the variance of the empirical probability of misidentification $\widehat{p}$ is $\widehat{p}(1 - \widehat{p})$. Thus, $\widehat{p}$ provides sufficient information about its distribution, so we do not show other graphs such as box plots and confidence intervals.

According to our results and existing studies, we theoretically expect the highest performance from the GJ algorithm, followed by the GNA algorithm in large samples. Our GNA algorithm follows these, with the other algorithms trailing.

Our empirical results align with theoretical expectations. The GNA algorithm tends to outperform the other methods, while the GJ is usually the best, which is an oracle algorithm, Additionally, as the variances are largeer, the GNA algorithm outperfoms the other more significantly. Interestingly, the GNA algorithm often outperforms the GJ algorithm. In finite samples, we observe cases where the SR and Uniform algorithms outperform the GJ and GNA algorithms. However, in large samples, the GJ and GNA algorithms tend to outperform them. This result implies that our proposed algorithms are asymptotically optimal but can be suboptimal in finite samples.

\section{Conclusion}
\label{sec:conclusion}

This study develops the GNA algorithm for the fixed-budget BAI problem. We demonstrate that the upper bound on its probability of misidentification aligns with the worst-case lower bound under the small-gap regime, providing an answer to a longstanding open problem regarding the existence of optimal algorithms in fixed-budget BAI, at least within a restricted distribution class. By restricting the class of distributions to the small-gap regime, we address the challenges posed by alternative lower bounds identified in existing studies \citep{Ariu2021, degenne2023existence}. 

The GNA algorithm represents a significant advancement in the theory of BAI. As a generalization of the Neyman allocation \citep{Neyman1934OnTT}, it extends the applicability of this classical method beyond two-armed cases to multi-armed cases. By incorporating adaptive estimation and leveraging properties of the small-gap regime, our approach provides a natural and theoretically grounded solution to this extension. Furthermore, we resolve unresolved technical issues highlighted in previous works, such as those in \citet{Kasy2021,Kasy2021corrigendum}, by demonstrating that the differences in KL divergence terms become negligible under the small-gap regime.

In addition to its theoretical contributions, the simplicity and adaptability of the GNA algorithm make it a promising approach for practical applications. The closed-form allocation rule derived under the small-gap regime offers computational advantages, enabling efficient implementation without requiring Bayesian methods or extensive computational resources.

For future work, it would be interesting to investigate the finite-sample properties of the GNA algorithm and explore finite-sample optimal algorithms. While our results establish asymptotic optimality, practical applications often operate under finite budgets where deviations from the theoretical results may occur. Investigating algorithms that optimize performance in such settings or incorporating robust allocation strategies could further enhance the practical utility of the GNA algorithm. \citet{Kato2020adaptive} examine the finite-sample properties of the A2IPW algorithm using the law-of-iterated logarithms, which has been refined by \citet{cook2023semiparametric}, using the concentration inequalities proposed by \citet{Balsubramani2016} and \citet{Howard2020TimeuniformNN}. In the study of finite-sample optimal algorithms, the lower bound by \citet{Carpentier2016} is particularly insightful, as it has already been used in existing studies \citep{yang2022minimax,Ariu2021,degenne2023existence}.

Another potential research direction is extending the analysis to cases with larger gaps. Additionally, exploring connections between the small-gap regime and other adaptive experimental designs could provide deeper insights into the generalizability of our approach.

In summary, this study contributes a theoretically optimal and practically implementable algorithm for fixed-budget BAI, offering both an answer to open theoretical questions and a foundation for future advancements in adaptive experimentation and BAI research.

\bibliographystyle{tmlr}
\bibliography{ReStud.bbl}

\onecolumn

\appendix

\section{Proof of Theorem~\ref{thm:lower_bound}}
\label{appdx:lower_proof}
This section provides the proof for Theorem~\ref{thm:lower_bound}. Our proof is inspired by \citet{Kaufman2016complexity}, \citet{Garivier2016}, and \citet{kato2024worstcase}.

\subsection{Transportation Lemma}
Let us denote the number of allocated arms by 
\[N_a(T) = \sum^T_{t=1}\mathbbm{1}[A_t = a].\] 
First, we introduce the \emph{transportation} lemma, shown by \citet{Kaufman2016complexity}. 
\begin{proposition}[Transportation lemma. From Lemma~1 in \citet{Kaufman2016complexity}]
\label{prp:transport}
Let $P$ and $Q$ be two bandit
models with $K$ arms such that for all $a$, the distributions $P_a$ and $Q_a$ of $Y_a$ are mutually absolutely continuous. 
Then, we have
\[
\sum_{a=1}^{K} \mathbb{E}_P[N_a(T)] \mathrm{KL}(P_a,Q_a)\geq
\sup_{\mathcal{E} \in \mathcal{F}_T} \ d(\mathbb{P}_P(\mathcal{E}),\mathbb{P}_{Q}(\mathcal{E})),
\]
where $d(x,y):= x\log (x/y) + (1-x)\log((1-x)/(1-y))$ is the binary relative entropy, with the convention 
that $d(0,0)=d(1,1)=0$.
\end{proposition}
Here, $Q$ corresponds to an alternative hypothesis that is used for deriving lower bounds and not an actual distribution. 

\subsection{KL Divergence and Fisher Information}
We recap the following well-known relationship that holds between the KL divergence and the Fisher information. This result is a cansequence of the Taylor expansion. 
\begin{proposition}[Proposition 15.3.2. in \citet{Duchi2023} and Theorem 4.4.4 in \citet{calin2014geometric}]
\label{prp:kl_fisher}
    For $P_{\mu_a, a}$ and $Q_{\nu_a, a}$ of $P, Q \in \mathcal{P}$, we have
    \begin{align}
        \lim_{\nu_a \to \mu_a}\frac{1}{\left(\mu_a - \nu_a\right)^2}\mathrm{kl}(\mu_a,\nu_a) = \frac{1}{2}I(\mu_a)
    \end{align}
\end{proposition}

\subsection{Proof of Theorem~\ref{thm:lower_bound}}
By using Propositions~\ref{prp:transport} and \ref{prp:kl_fisher}, we prove Theorem~\ref{thm:lower_bound} below. 

\begin{proof}[Proof of Theorem~\ref{thm:lower_bound}]
Let $\bm{\mu} = (\mu_a)_{a\in[K]}$ be the baseline mean outcomes of $Y_a$ under $P_{\bm{\mu}}$ whose best arm is fixed at $\widetilde{a} \in [K]$. Corresponding to the baseline mean outcomes $\bm{\mu}$, let $\bm{\nu} = (\nu_a)_{a\in[K]}$ be an alternative mean outcomes of $Y_a$ under $P_{\bm{\nu}}$ whose best arm is \emph{not} $\widetilde{a}$. 

Let $\mathcal{E}$ be the event $\widehat{a}^\pi_T = a^\star(P_{\bm{\nu}}) \neq \widetilde{a}$. 
Between the baseline distribution $P_{\bm{\mu}}$ and an alternative hypothesis $P_{\bm{\nu}}$, from Proposition~\ref{prp:transport}, we have
\[
\sum_{a=1}^{K} \mathbb{E}_P[N_a(T)] \mathrm{KL}(P_{a, \mu_a},P_{a, \nu_a})\geq
\sup_{\mathcal{E} \in \mathcal{F}_T} \ d(\mathbb{P}_{P_{\bm{\mu}}}(\mathcal{E}),\mathbb{P}_{P_{\bm{\nu}} }(\mathcal{E})).
\]

Under any consistent algorithm $\pi \in \Pi^{\mathrm{const}}$, we have $\mathbb{P}_{P_{\bm{\mu}}}(\mathcal{E}) \to 0$ and $\mathbb{P}_{P_{\bm{\nu}} }(\mathcal{E}) \to 1$ as $T \to \infty$. 

Therefore, for any $\varepsilon > 0$, there exists $T(\epsilon)$ such that for all $T \geq T(\varepsilon)$, it holds that
\[0\leq \mathbb{P}_{P_{\bm{\mu}}}(\mathcal{E}) \leq \varepsilon \leq \mathbb{P}_{P_{\bm{\nu}} }(\mathcal{E}) \leq 1.\]

Since $d(x,y)$ is defined as $d(x,y):= x\log (x/y) + (1-x)\log((1-x)/(1-y))$, we have
\begin{align*}
    &\sum_{a=1}^{K} \mathbb{E}_P[N_a(T)] \mathrm{KL}(P_{a, \mu_a},P_{a, \nu_a}) \geq d(\varepsilon,\mathbb{P}_{P_{\bm{\nu}} }(\mathcal{E}))\\
    &\ \ \ = \varepsilon\log \left(\frac{\varepsilon}{\mathbb{P}_{P_{\bm{\nu}} }(\mathcal{E})}\right) + \left(1 - \varepsilon\right)\log \left(\frac{1 - \varepsilon}{1 - \mathbb{P}_{P_{\bm{\nu}} }(\mathcal{E})}\right)\\
    &\ \ \ \geq \varepsilon\log \left(\varepsilon\right) + \left(1 - \varepsilon\right)\log \left(\frac{1 - \varepsilon}{1 - \mathbb{P}_{P_{\bm{\nu}} }(\mathcal{E})}\right)\\
    &\ \ \ \geq \varepsilon\log \left(\varepsilon\right) + \left(1 - \varepsilon\right)\log \left(\frac{1 - \varepsilon}{\mathbb{P}_{P_{\bm{\nu}} }(\widehat{a}^\pi_T \neq a^\star(P_{\bm{\nu}})}\right).
\end{align*}
Note that $\varepsilon$ is closer to $\mathbb{P}_{P_{\bm{\nu}} }(\mathcal{E})$ than $\mathbb{P}_{P_{\bm{\mu}}}(\mathcal{E})$; therefore, we used 
$d(\mathbb{P}_{P_{\bm{\mu}}}(\mathcal{E}),\mathbb{P}_{P_{\bm{\nu}} }(\mathcal{E})) \geq d(\varepsilon,\mathbb{P}_{P_{\bm{\nu}} }(\mathcal{E}))$. 

We divide both sides by $T$, take $\limsup_{T \to \infty}$, and let $\epsilon$ go to zero. Then, we obtain 
\begin{align*}
    &\limsup_{T \to \infty} - \frac{1}{T}\log \mathbb{P}_{P_{\bm{\nu}}}\big(\widehat{a}^\pi_T \neq a^\star(P_{\bm{\nu}})\big) \leq \limsup_{T \to \infty}\sum_{a\in[K]} \kappa^\pi_{a, T}(P_{\bm{\mu}})\mathrm{KL}(P_{a, \mu_a}, P_{a, \nu_a}),
\end{align*}
where $\kappa^\pi_{a, T}(P_{\bm{\mu}}) \coloneqq \frac{1}{T}\mathbb{E}_{P_{\bm{\mu}}}\left[N_a\right]$.\footnote{The proof of this part is inspired by that for Theorem~12 in \citet{Kaufman2016complexity}. A reader of our paper gave us a comment that ``this theorem is erroneous. Properly using Proposition A.1 yields $\sum_{a \in [K]} \mathbb{E}_Q[N_{T,a}] \text{KL}(Q_a, P_a) \geq - \mathbb{P}_Q(a_T \neq a^\star(P)) \log \mathbb{P}_P(a_T \neq a^\star(P)) - \log 2$. Using the consistency of the algorithm, we have $\mathbb{P}_Q(a_T \neq a^\star(P)) \to 0$ as $T \to \infty$. Therefore, the roles of $P$ and $Q$ should be reversed.'' (We keep the original notation in the previous comment, which is different from ours, since it is identical to the notations in \citet{Kaufman2016complexity}.) However, this comment is based on the confusion of the definition of the consistent algorithm and the event considered in the change of measure arguments (Proposition~\ref{prp:transport}), and our theorem holds. Recall that any consistent algorithm returns the true best arm with probability one under each distribution. Here, we defined the event in Proposition~\ref{prp:transport} as $\widehat{a}^\pi_T = a^\star(P_{\bm{\nu}})$, which is different from that in \citet{Kaufman2016complexity} since \citet{Kaufman2016complexity} defines the event as $\widehat{a}^\pi_T = a^\star(P_{\bm{\mu}})$. If we follow the notation in the comments, our event is $\widehat{a}^\pi_T = a^\star(Q)$, not $\widehat{a}^\pi_T = a^\star(P)$. This means that we use the following fact in the proof: under any consistent algorithm, $\mathbb{P}_P(a_T \neq a^\star(Q)) \to 0$ and $\mathbb{P}_Q(a_T \neq a^\star(Q)) \to 1$ as $T \to \infty$.  Therefore, our roles of $P$ and $Q$ are correct, which is reversed from \citet{Kaufman2016complexity}. Similar proof techniques have also been employed in \citet{Komiyama2022} and \citet{degenne2023existence} and confirmed the soundness.} 

Then, taking $\inf_{P_{\bm{\nu}} \in \cup_{b\in[K]\backslash\{\widetilde{a}\}}\mathcal{P}\big(b, \theta^\star(, \overline{\Delta}\big)} $ in both sides, we obtain
\begin{align*}
    &\inf_{\substack{\bm{\nu} \in \Theta^K\colon\\ \argmax_{a\in[K]}\nu_a \neq \widetilde{a}}} \limsup_{T \to \infty} - \frac{1}{T}\log \mathbb{P}_{P_{\bm{\nu}}}\big(\widehat{a}^\pi_T \neq a^\star(P_{\bm{\nu}})\big)\\
    &
    \leq \inf_{\substack{\bm{\nu} \in \Theta^K\colon\\ \argmax_{a\in[K]}\nu_a \neq \widetilde{a}}} \limsup_{T \to \infty}\sum_{a\in[K]} \kappa^\pi_{a, T}(P_{\bm{\mu}})\mathrm{KL}(P_{a, \mu_a}, P_{a, \nu_a}).
\end{align*}

From Proposition~\ref{prp:kl_fisher}, for any $\varepsilon > 0$, there exists $\Xi_a(\varepsilon)$ such that for all $- \Xi_a(\varepsilon) < \xi_a \coloneqq - \mu_a + \nu_a < \Xi_a(\varepsilon)$, the following holds:
\begin{align}
    \mathrm{kl}(\mu_a, \mu_a + \xi_a) \leq \frac{\xi^2_a}{2}I\big(\mu_a\big) + \varepsilon \xi^2_a = \frac{\xi^2_a}{2\sigma_a(\mu_a)} + \varepsilon \xi^2_a,
\end{align}
where we used $I\big(\mu_a\big) = \sigma^2_a(\mu_a)$. 

Then, we have
\begin{align*}
        &\inf_{\substack{\bm{\nu} \in \Theta^K\colon\\ \argmax_{a\in[K]}\nu_a \neq \widetilde{a}}}\limsup_{T \to \infty} - \frac{1}{T}\log \mathbb{P}_{P_{\bm{\nu}}}(\widehat{a}^\pi_T \neq a^\star(P_{\bm{\nu}}))\\
        &\leq \inf_{\substack{\bm{\nu} \in \Theta^K\colon\\ \argmax_{a\in[K]}\nu_a \neq \widetilde{a}}}\limsup_{T \to \infty}\sum_{a\in[K]} \left\{\kappa^\pi_{a, T}(P_{\bm{\mu}})\frac{\left(\mu_a - \nu_a\right)^2}{2\sigma^2_a(\mu_a)} + \varepsilon \left(\mu_a - \nu_a\right)^2\right\}\\
        &\leq  \sup_{\bm{w}\in\mathcal{W}}\inf_{\substack{\bm{\nu} \in \Theta^K\colon\\ \argmax_{a\in[K]}\nu_a \neq \widetilde{a}}}\sum_{a\in[K]}\left\{ w_a\frac{\left(\mu_a - \nu_a\right)^2}{2\sigma^2_a(\mu_a)} + \varepsilon \left(\mu_a - \nu_a\right)^2\right\},
\end{align*}
where $\mathcal{W} \coloneqq \left\{\bm{w} \in [0, 1]^K \colon \sum_{a\in[K]}w_a = 1\right\}$. 

We compute the RHS as follows:
\begin{align*}
    &\inf_{\stackrel{(\nu_a)\in\Theta^K:}{\argmax_{a\in[K]}\nu_a \neq \widetilde{a}}}\sum_{a\in[K]}w_a\frac{\big(\mu_a - \nu_a\big)^2}{2\sigma^2_a(\mu_a)}\\
    &= \min_{a\in[K]\backslash\{\widetilde{a}\}}\inf_{\stackrel{(\nu_a)\in\Theta^K:}{\nu_a > \nu_{\widetilde{a}}}}\sum_{a\in[K]}w_a\frac{\big(\mu_a - \nu_a\big)^2}{2\sigma^2_a(\mu_a)}\\
    &= \min_{a\in[K]\backslash\{\widetilde{a}\}}\inf_{\stackrel{(\nu_{\widetilde{a}}, \nu_a)\in\Theta^2:}{\nu_a > \nu_{\widetilde{a}}}}\left\{w_{\widetilde{a}}\frac{\left(\nu_{\widetilde{a}} - \mu_{\widetilde{a}}\right)^2}{2\sigma^2_{\widetilde{a}}(\mu_{\widetilde{a}})} + w_a\frac{\left(\nu_a - \mu_a\right)^2}{2\sigma^2_a(\mu_a)}\right\}\\
    &= \min_{a\in[K]\backslash\{\widetilde{a}\}}\min_{\nu \in \left[\mu_a, \mu_{\widetilde{a}}\right]}\left\{w_{\widetilde{a}}\frac{\left(\nu - \mu_{\widetilde{a}}\right)^2}{2\sigma^2_{\widetilde{a}}(\mu_{\widetilde{a}})} + w_a\frac{\left(\nu - \mu_a\right)^2}{2\sigma^2_a(\mu_a)}\right\}.
\end{align*}
Then, by solving the optimization problem, we obtain 
\begin{align*}
    & \min_{a\in[K]\backslash\{\widetilde{a}\}}\min_{\nu \in \left[\mu_a, \mu_{\widetilde{a}}\right]}\left\{w_{\widetilde{a}}\frac{\left(\nu - \mu_{\widetilde{a}}\right)^2}{2\sigma^2_a(\mu_{\widetilde{a}})} + w_a\frac{\left(\nu - \mu_a\right)^2}{2\sigma^2_a(\mu_a)}\right\}\\
     &= \min_{a\in[K]\backslash\{\widetilde{a}\}}\frac{\left(\mu_{\widetilde{a}} - \mu_a\right)^2}{2\left(\frac{\sigma^2_{\widetilde{a}}(\mu_{\widetilde{a}})}{w_{\widetilde{a}}} + \frac{\sigma^2_a(\mu_a)}{w_{a}}\right)},
\end{align*}
where the optimizer $v^*$ in the inner minimization problem is 
\[\frac{1}{\left(\frac{\sigma^2_{\widetilde{a}}(\mu_{\widetilde{a}})}{w_{\widetilde{a}}} + \frac{\sigma^2_a(\mu_a)}{w_{a}}\right)}\left(\frac{\sigma^2_a(\mu_a)}{w_{a}}\mu_{\widetilde{a}} + \frac{\sigma^2_{\widetilde{a}}(\mu_{\widetilde{a}})}{w_{\widetilde{a}}}\mu_{a}\right).\]

Therefore, we have
\begin{align*}
&\lim_{\overline{\Delta}\to 0}\inf_{\stackrel{(\nu_a)\in\Theta^K:}{\argmax_{a\in[K]}\nu_a \neq \widetilde{a}}}\limsup_{T \to \infty} - \frac{1}{\overline{\Delta}^2T}\log \mathbb{P}_{P_{\bm{\nu}}}(\widehat{a}^\pi_T \neq a^\star(P_{\bm{\mu}}))\\
&\leq \sup_{\bm{w}\in\mathcal{W}}\min_{a\in[K]\backslash\{\widetilde{a}\}}\frac{\left(\mu_{\widetilde{a}} - \mu_a\right)^2}{2\left(\frac{\sigma^2_{\widetilde{a}}(\mu_{\widetilde{a}})}{w_{\widetilde{a}}} + \frac{\sigma^2_a(\mu_a)}{w_{a}}\right)}.
\end{align*}

Lastly, we solve
\begin{align*}
\max_{w \in \mathcal{W}}\min_{a\neq b} \frac{1}{2\left(\frac{\sigma^2_{\widetilde{a}}(\mu_{\widetilde{a}})}{w_{\widetilde{a}}} + \frac{\sigma^2_a(\mu_a)}{w_{a}}\right)}.
\end{align*}

We solve the optimization problem by solving the following non-linear programming:
\begin{align*}
    &\max_{R > 0, \bm{w} = \{w_1,w_2\dots,w_K\} \in (0,1)^{K} }\ \ \ R\\
    \mathrm{s.t.}&\ \ \ R \left(\frac{\sigma^2_{\widetilde{a}}(\mu_{\widetilde{a}})}{w_{\widetilde{a}}} + \frac{\sigma^2_a(\mu_a)}{w_a}\right) \zeta - 1 \leq 0\qquad \forall a \in [K]\backslash\{\widetilde{a}\},\\
    &\ \ \ \sum_{a\in[K]}w_a - 1 = 0,\\
    &\ \ \ w_a > 0 \qquad \forall a\in [K].
\end{align*}
For simplicity, we denote $(\sigma^2_{a}(\mu_{a}))_{a\in[K]}$ by $(\sigma^2_a)_{a\in[K]}$.

Let $\bm{\lambda} = \{\lambda_a\}_{a\in[K]\backslash\{\widetilde{a}\}} \in [-\infty, 0]^{K-1}$ and $\gamma \geq 0$ be Lagrangian multipliers. Then, we define the following Lagrangian function:
\begin{align*}
    &L(\bm{\lambda}, \bm{\gamma}; R, \bm{w}) = R + \sum_{a\in [K]\backslash \{\widetilde{a}\}}\lambda_a \left(R \left(\frac{\sigma^2_{\widetilde{a}}}{w_{\widetilde{a}}} + \frac{\sigma^2_a}{w_a}\right) - 1\right) - \gamma\left(\sum_{a\in[K]}w_a - 1\right).
\end{align*}
Note that the objective ($R$) and constraints ($R \left(\frac{\sigma^2_{\widetilde{a}}}{w_{\widetilde{a}}} + \frac{\sigma^2_a}{w_a}\right)  - 1 \leq 0$ and $\sum_{a\in[K]}w_a - 1 = 0$) are differentiable convex functions for $R$ and $\bm{w}$. 

Here, the global optimizer $R^\dagger$ and $\bm{w}^\dagger = \{w^\dagger_a\} \in (0,1)^{K}$ satisfies the following KKT conditions:
\begin{align}
\label{eq:cond11}
&1 +  \sum_{a\in[K]\backslash\{b\}}\lambda^\dagger_a\left(\frac{\sigma^2_{\widetilde{a}}}{w^\dagger_{\widetilde{a}}} + \frac{\sigma^2_a}{w^\dagger_a}\right) = 0\\
\label{eq:cond21}
&-2\sum_{a\in [K]\backslash\{\widetilde{a}\}}\lambda^\dagger_aR^\dagger\frac{\sigma^2_{\widetilde{a}}}{(w^\dagger_{\widetilde{a}})^2}  = \gamma^\dagger\\
\label{eq:cond41}
&-2\lambda^\dagger_aR^\dagger\frac{\sigma^2_a}{(w^\dagger_a)^2} = \gamma^\dagger\qquad \forall a \in [K]\backslash\{b\}\\
\label{eq:cond31}
&\lambda^\dagger_a \left(R^\dagger\left(\frac{\sigma^2_{\widetilde{a}}}{w^\dagger_{\widetilde{a}}} + \frac{\sigma^2_a}{w^\dagger_a}\right) - 1\right) = 0\qquad \forall a \in [K]\backslash\{\widetilde{a}\}\\
&\gamma^\dagger \left(\sum_{c\in[K]}w^\dagger(c) - 1\right) = 0\nonumber\\
&\lambda^\dagger_a\leq 0\qquad \forall a \in [K]\backslash\{\widetilde{a}\}\nonumber.
\end{align}

Here, \eqref{eq:cond11} implies that there exists 
$a\in[K]\backslash\{\widetilde{a}\}$ such that $\lambda^\dagger_a < 0$ holds. This is because if $\lambda^\dagger_a = 0$ for all $a\in [K]\backslash \{\widetilde{a}\}$, $1 + 0 = 1 \neq 0$.

With $\lambda^\dagger_a < 0$, since $- \lambda^\dagger_aR^\dagger\frac{\sigma^2_a}{(w^\dagger_a)^2} > 0$ for all $a\in[K]$, it follows that $\gamma^\dagger > 0$. This also implies that $\sum_{c\in[K]}w^{c\dagger} - 1 = 0$. 

Then, \eqref{eq:cond31} implies that 
\begin{align*}
    R^\dagger\left(\frac{\sigma^2_{\widetilde{a}}}{w^\dagger_{\widetilde{a}}} + \frac{\sigma^2_a}{w^\dagger_a}\right) = 1 \qquad \forall a \in [K]\backslash\{\widetilde{a}\}.
\end{align*}
Therefore, we have
\begin{align}
\label{eq:cond511}
    \frac{\sigma^2_a}{w^\dagger_a} =  \frac{\sigma^2_c}{w^\dagger_c} \qquad \forall a,c \in [K]\backslash\{\widetilde{a}\}.
\end{align}
Let $\frac{\sigma^2_a}{w^\dagger_a} =  \frac{\sigma^2_c}{w^\dagger_c} = \frac{1}{R^\dagger} - \frac{\sigma^2_c}{w^\dagger_c} = U$. 
From \eqref{eq:cond511} and \eqref{eq:cond11}, 
\begin{align}
\label{eq:cond61}
    \sum_{c\in [K]\backslash\{\widetilde{a}\}}\lambda^\dagger_c = - \frac{1}{\frac{\sigma^2_c}{w^\dagger_c} + U}
\end{align}
From \eqref{eq:cond21} and \eqref{eq:cond41}, we have
\begin{align}
\label{eq:cond71}
&\frac{\sigma^2_{\widetilde{a}}}{(w^\dagger_{\widetilde{a}})^2}\sum_{c\in [K]\backslash\{\widetilde{a}\}}\lambda^\dagger_c = \lambda^\dagger_a\frac{\sigma^2_a}{(w^\dagger_a)^2}\qquad \forall a \in [K]\backslash\{\widetilde{a}\}.
\end{align}
From \eqref{eq:cond61} and \eqref{eq:cond71}, we have
\begin{align}
\label{eq:cond81}
    -\frac{\sigma^2_{\widetilde{a}}}{(w^\dagger_{\widetilde{a}})^2} = \lambda^\dagger_a\frac{\sigma^2_a}{(w^\dagger_a)^2}\left(\frac{\sigma^2_{\widetilde{a}}}{w^\dagger_{\widetilde{a}}} + U\right)\qquad \forall a \in [K]\backslash\{\widetilde{a}\}.
\end{align}
From \eqref{eq:cond11} and \eqref{eq:cond81}, we have
\begin{align*}
    w^\dagger_{\widetilde{a}} = \sqrt{\sigma^2_{\widetilde{a}}\sum_{a\in[K]\backslash\{\widetilde{a}\}}\frac{(w^\dagger_a)^2}{\sigma^2_a}}.
\end{align*}

In summary, the KKT conditions are given as follows:
\begin{align*}
&w^\dagger_{\widetilde{a}} = \sqrt{\sigma^2_{\widetilde{a}}\sum_{a\in[K]\backslash\{\widetilde{a}\}}\frac{(w^\dagger_{a})^2}{\sigma^2_a}}\\
&\frac{\sigma^2_{\widetilde{a}}}{(w^\dagger_{\widetilde{a}})^2} = - \lambda^\dagger_a\frac{\sigma^2_a}{(w^\dagger_a)^2}\left(\left(\frac{\sigma^2_{\widetilde{a}}}{w^\dagger_{\widetilde{a}}} + \frac{\sigma^2_a}{w^\dagger_a}\right)\right)\qquad \forall a \in [K]\backslash\{\widetilde{a}\}\\
&-\lambda^\dagger_a\frac{\sigma^2_a}{(w^\dagger_a)^2} = \widetilde{\gamma}^\dagger\qquad \forall a \in [K]\backslash\{\widetilde{a}\}\\
&\frac{\sigma^2_a}{w^\dagger_a}= \frac{1}{R^\dagger} - \frac{\sigma^2_{\widetilde{a}}}{w^\dagger_{\widetilde{a}}} \qquad \forall a\in [K]\backslash\{\widetilde{a}\} \\
&\sum_{a\in[K]}w^\dagger_a = 1\nonumber\\
&\lambda^\dagger_a\leq 0\qquad \forall a \in [K]\backslash\{\widetilde{a}\},
\end{align*}
where $\widetilde{\gamma}^\dagger = \gamma^\dagger/2 R^\dagger$. 

From $w^\dagger_{b} = \sqrt{\sigma^2_{b}\sum_{a\in[K]\backslash\{\widetilde{a}\}}\frac{(w^\dagger_a)^2}{\sigma^2_a}}$ and $-\lambda^\dagger_a\frac{\sigma^2_a}{(w^\dagger_a)^2} = \widetilde{\gamma}^\dagger$, we have
\begin{align*}
    &w^\dagger_{\widetilde{a}} = \sigma_{\widetilde{a}}\sqrt{\sum_{a\in[K]\backslash\{\widetilde{a}\}}-\lambda^\dagger_a}/\sqrt{\widetilde{\gamma}^\dagger}\\
    &w^\dagger_a = \sqrt{-\lambda^\dagger_a/\widetilde{\gamma}^\dagger}\sigma_a.
\end{align*}
From $\sum_{a\in[K]}w^\dagger_a = 1$, we have
\begin{align*}
\sigma_b\sqrt{\sum_{a\in[K]\backslash\{\widetilde{a}\}}-\lambda^\dagger_a}/\sqrt{\widetilde{\gamma}^\dagger} + \sum_{a\in[K]\backslash\{\widetilde{a}\}}\sqrt{-\lambda^\dagger_a/\widetilde{\gamma}^\dagger}\sigma_a = 1.
\end{align*}
Therefore, the following holds:
\begin{align*}
    \sqrt{\widetilde{\gamma}^\dagger} = \sigma_{\widetilde{a}}\sqrt{\sum_{a\in[K]\backslash\{\widetilde{a}\}}-\lambda^\dagger_a} + \sum_{a\in[K]\backslash\{\widetilde{a}\}}\sqrt{-\lambda^\dagger_a}\sigma_a. 
\end{align*}
Hence, the allocation ratio is computed as
\begin{align*}
    &w^\dagger_{\widetilde{a}} = \frac{\sigma_{\widetilde{a}}\sqrt{\sum_{a\in[K]\backslash\{\widetilde{a}\}}-\lambda^\dagger_a}}{\sigma_{\widetilde{a}}\sqrt{\sum_{a\in[K]\backslash\{\widetilde{a}\}}-\lambda^\dagger_a} + \sum_{a\in[K]\backslash\{\widetilde{a}\}}\sqrt{-\lambda^\dagger_a}\sigma_a}\\
    &w^\dagger_a = \frac{\sqrt{-\lambda^\dagger_a}\sigma_a}{\sigma_{\widetilde{a}}\sqrt{\sum_{a\in[K]\backslash\{\widetilde{a}\}}-\lambda^\dagger_a} + \sum_{a\in[K]\backslash\{\widetilde{a}\}}\sqrt{-\lambda^\dagger_a}\sigma_a},
\end{align*}
where from $\frac{\sigma^2_{\widetilde{a}}}{(w^\dagger_{\widetilde{a}})^2} = - \lambda^\dagger_a\frac{\sigma^2_a}{(w^\dagger_a)^2}\left(\frac{\sigma^2_{b}}{w^\dagger_{b}} + \frac{\sigma^2_a}{w^\dagger_a}\right)$, $(\lambda^\dagger_a)_{a\in[K]\backslash\{\widetilde{a}\}}$ satisfies,
\begin{align*}
    &\frac{1}{\sum_{a\in[K]\backslash\{\widetilde{a}\}}-\lambda^\dagger_a}\\
    &= \left(\frac{\sigma_{\widetilde{a}}}{\sqrt{\sum_{a\in[K]\backslash\{\widetilde{a}\}}-\lambda^\dagger_a}} + \frac{\sigma_a}{\sqrt{-\lambda^\dagger_a}}\right)\left(\sigma_{\widetilde{a}}\sqrt{\sum_{c\in[K]\backslash\{\widetilde{a}\}}-\lambda^{c\dagger}} + \sum_{c\in[K]\backslash\{\widetilde{a}\}}\sqrt{-\lambda^{c\dagger}}\sigma^c_0\right)\\ 
    &= \left(\sigma_{\widetilde{a}} + \frac{\sigma_a}{\sqrt{-\lambda^\dagger_a}}\sqrt{\sum_{c\in[K]\backslash\{\widetilde{a}\}}-\lambda^{c\dagger}}\right)\left(\sigma_{\widetilde{a}} + \frac{\sum_{c\in[K]\backslash\{\widetilde{a}\}}\sqrt{-\lambda^{c\dagger}}\sigma_c}{\sum_{c\in[K]\backslash\{\widetilde{a}\}}-\lambda^{c\dagger}}\sqrt{\sum_{c\in[K]\backslash\{\widetilde{a}\}}-\lambda^{c\dagger}}\right).
\end{align*}

Then, the following solutions satisfy the above KKT conditions: 
\begin{align*}
    R^\dagger & \left(\sigma_{\widetilde{a}} + \sqrt{\sum_{a\in[K]\backslash\{b\}}\sigma^2_a}\right)^2 = 1\\
    w^\dagger_{\widetilde{a}} & = \frac{\sigma_{\widetilde{a}}\sqrt{\sum_{a\in[K]\backslash\{\widetilde{a}\}}\sigma^2_a}}{\sigma_{\widetilde{a}}\sqrt{\sum_{a\in[K]\backslash\{\widetilde{a}\}}\sigma^2_a} + \sum_{a\in[K]\backslash\{\widetilde{a}\}}\sigma^2_a}\\
    w^\dagger_a &= \frac{\sigma^2_a}{\sigma_{\widetilde{a}}\sqrt{\sum_{a\in[K]\backslash\{\widetilde{a}\}}\sigma^2_a} + \sum_{a\in[K]\backslash\{\widetilde{a}\}}\sigma^2_a}\\
    \lambda^\dagger_a &= - \sigma^2_a\\
    \gamma^\dagger &= \left(\sigma_{\widetilde{a}}\sqrt{\sum_{a\in[K]\backslash\{\widetilde{a}\}}\sigma^2_a} + \sum_{a\in[K]\backslash\{\widetilde{a}\}}\sigma^2_a\right)^2.
\end{align*}

Therefore, given $\widetilde{a} \in [K]$ and $\bm{\mu}$, we have
\begin{align*}
    &\limsup_{0 < \underline{\Delta} < \overline{\Delta} \to +0}\inf_{P\in\mathcal{P}\big(\underline{\Delta}, \overline{\Delta}\big)}\limsup_{T \to \infty} -\frac{1}{\overline{\Delta}^2T}\log \mathbb{P}_{P}\left(\widehat{a}^\pi_T \neq a^\star(P)\right)\leq V(\mu_{\widetilde{a}}),
\end{align*}
where recall that
\begin{align*}
    V(\widetilde{a}, \mu_{\widetilde{a}}) &= \frac{1}{2\left(\sigma_{a}(\mu_{\widetilde{a}}) + \sqrt{\sum_{b\in[K]\backslash\{a\}}\sigma^2_b(\mu_{\widetilde{a}})}\right)^2}.
\end{align*}

We can choose $\widetilde{a} \in [K]$ and $\bm{\mu}$ by our selves; therefore, by taking their worst case, we have
\begin{align*}
    &\limsup_{0 < \underline{\Delta} < \overline{\Delta} \to +0}\inf_{P\in\mathcal{P}\big(\underline{\Delta}, \overline{\Delta}\big)}\limsup_{T \to \infty} -\frac{1}{\overline{\Delta}^2T}\log \mathbb{P}_{P}\left(\widehat{a}^\pi_T \neq a^\star(P)\right)\leq V^*,
\end{align*}
where recall that
\begin{align*}
    V^* &= \min_{\widetilde{a}\in[K]}\min_{\mu \in \Theta} V(\widetilde{a}, \mu). 
\end{align*}

This completes the proof. 
\end{proof}

\section{Proof of Lemma~\ref{lem:upper_bound}}
\label{appdx:upper_proof}
To show Lemma~\ref{lem:upper_bound}, we derive an upper bound of \[\mathbb{P}_{P_0}\left(\widehat{\mu}^{\mathrm{A2IPW}}_{a^\star(P_0),T} \leq \widehat{\mu}^{\mathrm{A2IPW}}_{a, T}\right)\]
for $a\in[K]\backslash \{a^\star(P_0)\}$.
The bound is stated in the following lemma.
We show the proof in Appendix~\ref{appdx:lem_upper_proof}. We show the result for a bandit model $\mathcal{P}((\Delta_a)_{a\in[K]})$ with finite mean and variance, while for all $P\in\mathcal{P}((\Delta_a)_{a\in[K]})$, $\mu_{a^\star(P)} - \mu_a(P)$ is given as $\Delta_a > 0$. This class is wider than $\mathcal{P}\big(\underline{\Delta}, \overline{\Delta}\big)$. We also define $\mathcal{P}$ as a bandit model with finite mean and variance.

\begin{lemma}[Probability of misidentification of the A2IPW estimator]
\label{lem:optimal}
Let $\mathcal{P}((\Delta_a)_{a\in[K]})$ be a bandit model with finite mean and variance, whose variance is given as 
$(\sigma^2_a)_{a\in[K]}$ under $P_0$ and gaps are lower bounded by $(\Delta_a)_{a\in[K]}$. For each each $a\in[K]\backslash\{a^\star(P_0)\}$ and for any $\epsilon > 0$, there exist  $0 < \Delta^*_a < \Delta_0(\epsilon)$, the following holds: there exists $T_0(\Delta^*_a, \epsilon)$ such that for all $T > T_0(\Delta^*_a, \epsilon)$, it holds that 
\begin{align*}
       &\mathbb{P}_{P_0}\left(\widehat{\mu}^{\mathrm{A2IPW}}_{a^\star(P_0), T} \leq \widehat{\mu}^{\mathrm{A2IPW}}_{a, T}\right) \leq \exp\left( - \frac{T\Delta^{*2}_a}{2\left(\frac{\sigma^2_{a^\star(P_0)}}{w^{\mathrm{GNA}}_{a^\star(P_0)}} + \frac{\sigma^2_a}{w^{\mathrm{GNA}}_a}\right)} + \epsilon T\Delta^{*2}_a\right).
    \end{align*}
for all $P_0 \in \underline{\mathcal{P}}\left((\Delta_b)_{b\in[K]}\right)$ such that $\Delta_a \leq \Delta^*_a$.
\end{lemma}

Then, we prove Lemma~\ref{lem:upper_bound} as follows:
\begin{proof}
We have 
\begin{align*}
     & - \frac{1}{T}\log \mathbb{P}_{P_0}
     \left(\widehat{a}^{\mathrm{GNA}}_T \neq a^\star(P_0)\right)\\
     &= - \frac{1}{T}\log \sum_{a\neq a^\star(P_0)}\mathbb{P}_{P_0}\left(\widehat{\mu}^{\mathrm{A2IPW}}_{a^\star(P_0), T} \leq \widehat{\mu}^{\mathrm{A2IPW}}_{a, T}\right)  
     \\
     &\geq - \frac{1}{T}\log\left\{ (K-1) \max_{a\neq a^\star(P_0)} \mathbb{P}_{P_0}\left(\widehat{\mu}^{\mathrm{A2IPW}}_{a^\star(P_0), T} \leq \widehat{\mu}^{\mathrm{A2IPW}}_{a, T}\right)  \right\}.
\end{align*}
From Lemma~\ref{lem:optimal}, for any $\epsilon > 0$, there exist  $0 < \Delta < \Delta_0(\epsilon)$, the following holds: there exists $T_0(\Delta, \epsilon)$ such that for all $T > T_0(\Delta, \epsilon)$, it holds that 
\begin{align*}
    - \frac{1}{T}\log \mathbb{P}_{P_0}\left(\widehat{\mu}^{\mathrm{A2IPW}}_{a^\star(P_0), T} \geq \widehat{\mu}^{\mathrm{A2IPW}}_{a, T}\right) \geq \frac{\Delta^2}{2\left(\frac{\sigma^2_{a^\star(P_0)}}{w^{\mathrm{GNA}}_{a^\star(P_0)}} + \frac{\sigma^2_a}{w^{\mathrm{GNA}}_a}\right)} - \epsilon\Delta^2.
\end{align*}
for all $P_0 \in \mathcal{P}\big(\underline{\Delta}, \overline{\Delta}\big)$ such that $\underline{\Delta} < \overline{\Delta} < \Delta$. 

Here, note that $ (K-1) $ can be asymptotically ignorable, and 
\[\frac{\sigma^2_{a^\star(P_0)}}{w^{\mathrm{GNA}}_{a^\star(P_0)}} + \frac{\sigma^2_a}{w^{\mathrm{GNA}}_a} = 2\left(\sigma_{a^\star(P_0)} + \sqrt{\sum_{a\in[K]\backslash\{a^\star(P_0)\}}\sigma^2_a}\right)^2\]
holds. 

Therefore, for any $\epsilon > 0$, there exist  $0 < \Delta < \Delta_0(\epsilon)$, the following holds: there exists $T_0(\Delta, \epsilon)$ such that for all $T > T_0(\Delta, \epsilon)$, it holds that 
\begin{align*}
     & - \frac{1}{T}\log \mathbb{P}_{P_0}
     \left(\widehat{a}^{\mathrm{GNA}}_T \neq a^\star(P_0)\right)\\
     &\geq \frac{ \Delta^2}{2\left(\sigma_{a^\star(P_0)} + \sqrt{\sum_{a\in[K]\backslash\{a^\star(P_0)\}}\sigma^2_a}\right)^2} - \epsilon \Delta^2\\
     &\geq \frac{ \underline{\Delta}^2}{2\left(\sigma_{a^\star(P_0)} + \sqrt{\sum_{a\in[K]\backslash\{a^\star(P_0)\}}\sigma^2_a}\right)^2} - \epsilon \underline{\Delta}^2.
\end{align*}
for all $P_0 \in \mathcal{P}\big(\underline{\Delta}, \overline{\Delta}\big)$ such that $\underline{\Delta} < \overline{\Delta} < \Delta$. 
\end{proof}

\section{Proof of Lemma~\ref{lem:optimal}}
\label{appdx:lem_upper_proof}
Let us define
\begin{align*}
    &\Psi_{a, t} \coloneqq \frac{1}{\sqrt{V(a)}}\Bigg(\frac{\mathbbm{1}[A_t = a^\star(P_0)]\big(Y_{a^\star(P_0), t}- \widehat{\mu}_{a^\star(P_0), t}\big)}{\widehat{w}^{\mathrm{GNA}}_{a^\star(P_0), t}} - \frac{\mathbbm{1}[A_t = a]\big(Y_{a, t}- \widehat{\mu}_{a, t}\big)}{\widehat{w}^{\mathrm{GNA}}_{a,t}}
    \\
    &\ \ \ \ \ \ \ \ \ \ \ \ \ \ \ \ \ \ \ \ \ \ \ \ \ \ \ \ \ \ \ \ \ \ \ \ \ \ \ \ \ \ \ \ \ \ \ \ \ \ \ \ \ \ \ \ \ \ \ \ \ \ \ \ \ \ \ \ \ \ \ \ \ \ \ \ \ \ \ \ \ \ \ \ \ \ \ \ \ \ \ \ \ + \widehat{\mu}_{a^\star(P_0), t} -  \widehat{\mu}_{a, t} - \Delta_a(P_0)\Bigg),
\end{align*}
where $\Delta_a(P_0) \coloneqq \mu_{a^\star(P_0)}(P_0) - \mu_a(P_0)$, and 
\[V(a) \coloneqq \frac{\sigma^2_{a^\star(P_0)}}{w^{\mathrm{GNA}}_{a^\star(P_0)}} + \frac{\sigma^2_a}{w^{\mathrm{GNA}}_a}.\] 
Then, we have 
\[\widehat{\mu}^{\mathrm{A2IPW}}_{a^\star(P_0), T} - \widehat{\mu}^{\mathrm{A2IPW}}_{a, T} = \frac{1}{T} \sum^T_{t=1}\Psi_{a, t}.\]
By using this result, we aim to derive the upper bound of 
\[\mathbb{P}_{P_0}\left(\widehat{\mu}^{\mathrm{A2IPW}}_{a^\star(P_0), T} \leq \widehat{\mu}^{\mathrm{A2IPW}}_{a, T}\right) = \mathbb{P}_{P_0}\left(\sum^T_{t=1}\Psi_{a, t} \leq - \frac{T\Delta_a(P_0)}{\sqrt{V(a)}}\right).\] 

Here, we show that $\left\{\Psi_{a, t}\right\}_{t\in[T]}$ is a martingale difference sequence (MDS). For each $t \in [T]$, it holds that
\begin{align*}
    &\sqrt{V(a)}\mathbb{E}_{P_0}\left[\Psi_{a, t}\mid\mathcal{F}_{t-1}\right]\\
    &=\mathbb{E}_{P_0}\left[\frac{\mathbbm{1}[A_t = a^\star(P_0)]\big(Y_{a^\star(P_0), t}- \widehat{\mu}_{a^\star(P_0), t}\big)}{\widehat{w}^{\mathrm{GNA}}_{a^\star(P_0),t}} + \widehat{\mu}_{a^\star(P_0), t}| \mathcal{F}_{t-1}\right]\\
    &\ \ \ \ \ - \mathbb{E}_{P_0}\left[\frac{\mathbbm{1}[A_t = a]\big(Y_{a, t}- \widehat{\mu}_{a, t}\big)}{\widehat{w}^{\mathrm{GNA}}_{a,t}} + \widehat{\mu}_{a, t}| \mathcal{F}_{t-1}\right] - \Delta_a(P_0)\\
    &= \frac{\widehat{w}^{\mathrm{GNA}}_{a^\star(P_0), t}\big(\mu_{a^\star(P_0)}(P_0)- \widehat{\mu}_{a^\star(P_0), t}\big)}{\widehat{w}^{\mathrm{GNA}}_{a^\star(P_0),t}} + \widehat{\mu}_{a^\star(P_0), t} - \frac{\widehat{w}^{\mathrm{GNA}}_{2, t}\big(\mu_a(P_0)- \widehat{\mu}_{a, t}\big)}{\widehat{w}^{\mathrm{GNA}}_{a,t}} + \widehat{\mu}_{a, t} - \Delta_a(P_0)\\
    &=\left(\mu_{a^\star(P_0)}(P_0) - \mu_a(P_0)\right) - \left(\mu_{a^\star(P_0)}(P_0) - \mu_a(P_0)\right)\\
    &= 0.
\end{align*}
This result implies that $\left\{\Psi_{a, t}\right\}_{t\in[T]}$ is an MDS. 

Since $\widehat{w}^{\mathrm{GNA}}_{a,t} > 0$, and the mean and variance of $Y_a$ are finite for all $P_0 \in \mathcal{P}$, the following lemma holds:
\begin{lemma}
    \label{lem:existence}
    For any $P_0 \in \mathcal{P}$, the first moment of $\Psi_{a, t}$ is zero, and the second moment of $\Psi_{a, t}$ exists.
\end{lemma}

Next, we prove the following lemma in Appendix...
\begin{lemma}[Almost sure convergence of the average of the second moment]
\label{lem:cesaro}
For any $P_0 \in \mathcal{P}$, it holds that 
\[\mathbb{P}_{P_0}\left(\lim_{T\to\infty}\frac{1}{T}\sum^T_{t=1}\left|\mathbb{E}_{P_0}\left[\Psi^2_{a, t}\mid\mathcal{F}_{t-1}\right] - 1\right| = 0\right) = 1\]
\end{lemma}
This result is a variant of the Ces\`{a}ro lemma for a case with almost sure convergence.

By using these lemma, we prove Lemma~\ref{lem:optimal}.
\begin{proof}[Proof of Lemma~\ref{lem:optimal}]
By applying the Chernoff bound, for any $v < 0$ and any $\lambda < 0$, it holds that
\begin{align}
\label{eq:target1}
    \mathbb{P}_{P_0}\left(\frac{1}{T}\sum^T_{t=1}\Psi_{a, t} \leq v\right) \leq \mathbb{E}_{P_0}\left[\exp\left(\lambda \sum^T_{t=1}\Psi_{a, t}\right)\right]\exp\left(-T\lambda v\right).
\end{align}
From the Chernoff bound and a property of an MDS, we have
\begin{align}
\label{eq:target2}
    &\mathbb{E}_{P_0}\left[\exp\left(\lambda \sum^T_{t=1}\Psi_{a, t}\right)\right]\nonumber\\
    &=\mathbb{E}_{P_0}\left[\prod^T_{t=1}\mathbb{E}_{P_0}\left[\exp\left(\lambda \Psi_{a, t}\right)\mid \mathcal{F}_{t-1}\right]\right]\nonumber\\
    &= \mathbb{E}_{P_0}\left[\exp\left(\sum^T_{t=1}\log \mathbb{E}_{P_0}\left[\exp\left(\lambda\Psi_{a, t}\right)\mid\mathcal{F}_{t-1}\right] \right)\right].
\end{align}

From the Taylor expansion around $\lambda = 0$, we obtain the following \citep[Section~1.2,][]{Doring2022}:
\begin{align}
\label{eq:taylor}
    &\lim_{\lambda \to 0}\frac{1}{\lambda^2}\log \mathbb{E}_{P_0}\left[\exp\left(\lambda\Psi_{a, t}\right)\mid\mathcal{F}_{t-1}\right] = \frac{1}{2} \mathbb{E}_{P_0}\left[\Psi^2_{a, t}\mid \mathcal{F}_{t-1}\right].
\end{align}
Here, we used Lemma~\ref{lem:existence}, which states that the conditional variance $\mathbb{E}_{P_0}\left[\Psi^2_{a, t}| \mathcal{F}_{t-1}\right]$ exists.

Then, from \eqref{eq:target1}, \eqref{eq:target2}, and \eqref{eq:taylor}, for any $\epsilon > 0$, there exist $\lambda_0(\epsilon)$ such that for all $0 < \lambda < \lambda_0(\epsilon)$, it holds that
\begin{align}
\label{eq:target3}
    &\mathbb{P}_{P_0}\left(\frac{1}{T}\sum^T_{t=1}\Psi_{a, t} \leq v\right)\nonumber\\
    &\leq \mathbb{E}_{P_0}\left[\exp\left(\lambda \sum^T_{t=1}\Psi_{a, t}\right)\right]\exp\left(-T\lambda v\right)\nonumber\\
    &\leq \mathbb{E}_{P_0}\left[\exp\left(\frac{\lambda^2}{2}\sum^T_{t=1}\mathbb{E}_{P_0}\left[\Psi^2_{a, t}\mid \mathcal{F}_{t-1}\right] + \lambda^2\epsilon\right)\right]\exp\left(-T\lambda v\right).
\end{align}

Let $v = \lambda$. From \eqref{eq:target3} and Lemma~\ref{lem:cesaro}, for any $\epsilon > 0$, there exist $\lambda_0(\epsilon) > 0$ such that for all $0 < \lambda < \lambda_0(\epsilon)$, the following holds: there exists $T_0(\lambda, \epsilon)$ such that for all $T > T_0(\lambda. \epsilon)$, it holds that
\begin{align*}
    &\mathbb{P}_{P_0}\left(\sum^T_{t=1}\Psi_{a, t} \leq Tv\right)\\
    &\leq \mathbb{E}_{P_0}\left[\exp\left(\frac{\lambda^2}{2}\sum^T_{t=1}\mathbb{E}_{P_0}\left[\Psi^2_{a, t}\mid \mathcal{F}_{t-1}\right] + \lambda^2\epsilon\right)\right]\exp\left(-T\lambda v\right)\\
    &= \mathbb{E}_{P_0}\left[\exp\left(\frac{\lambda^2}{2}\sum^T_{t=1}\mathbb{E}_{P_0}\left[\Psi^2_{a, t}\mid \mathcal{F}_{t-1}\right] + \lambda^2\epsilon - T\lambda^2\right)\right]\\
    &\leq \mathbb{E}_{P_0}\Bigg[\exp\Bigg(-\frac{T\lambda^2}{2} + \epsilon T \lambda^2 -  \left(\frac{\lambda^2}{2}\sum^T_{t=1} \mathbb{E}_{P_0}\left[\Psi^2_{a, t}| \mathcal{F}_{t-1}\right] - \frac{T\lambda^2}{2}\right)\Bigg) \Bigg]\\
    &\leq \mathbb{E}_{P_0}\left[\exp\left( -\frac{T\lambda^2}{2} + \frac{\lambda^2}{2}\sum^T_{t=1} \Big(\mathbb{E}_{P_0}\left[\Psi^2_{a, t}| \mathcal{F}_{t-1}\right] - 1\Big) + \epsilon T\lambda^2\right)\right]\\
    &= \exp\left(-\frac{T\lambda^2}{2}  + \epsilon T\lambda^2\right)\mathbb{E}_{P_0}\left[\exp\left(\frac{\lambda^2}{2}\sum^T_{t=1}\Big( \mathbb{E}_{P_0}\left[\Psi^2_{a, t}| \mathcal{F}_{t-1}\right] - 1\Big)\right)\right]\\
    &=  \exp\left(-\frac{T\lambda^2}{2}  + \epsilon T\lambda^2\right)\exp\left(\epsilon T\lambda^2\right)\\
    &= \exp\left(-\frac{T\lambda^2}{2}  + 2\epsilon T\lambda^2\right).
\end{align*}

Let $\lambda = - \frac{\Delta_a}{\sqrt{V(a)}_a}$. 
Therefore, for any $\epsilon > 0$, there exist  $0 < \Delta_a < \Delta_{a, 0}(\epsilon)$, the following holds: there exists $T_0(\Delta_a, \epsilon)$ such that for all $T > T_0(\Delta_a, \epsilon)$, it holds that 
    \begin{align*}
       &\mathbb{P}_{P_0}\left(\sum^T_{t=1}\Psi_{a, t} \leq Tv\right) \geq \exp\left(- \frac{T\Delta^2_a}{2V(a)} + 2\epsilon T\Delta^2_a\right).
    \end{align*}
for all $P_0 \in \underline{\mathcal{P}}\left((\Delta_a)_{a\in[K]}\right)$.
Thus, the proof is complete. 
\end{proof}

\section{Proof of Lemma~\ref{lem:cesaro}}
\label{sec:proof}
First, since there exists a constant $C > 0$ independent of $T$ such that $\sigma^2_a > C$, the arm allocation probability is strictly greater than zero. This ensures that each arm is allocated infinitely often. Therefore, from the law of large numbers, the following lemma holds.
\begin{lemma}
\label{lem:almost_sure_nuisance}
    For any $P_0\in \mathcal{P}$ and all $a\in[K]$, $\widehat{\mu}_{a, t}\xrightarrow{\mathrm{a.s}} \mu_{a}$ and $\widehat{\sigma}^2_{a, t} \xrightarrow{\mathrm{a.s}} \sigma^2_a$ as $t \to \infty$.
\end{lemma}
Furthermore, from $\widehat{\sigma}^2_{a, t} \xrightarrow{\mathrm{a.s}} \sigma^2_a$ and continuous mapping theorem, for all $a\in[K]$, $\widehat{w}^{\mathrm{GNA}}_{a,t}\xrightarrow{\mathrm{a.s}}w^{\mathrm{GNA}}_{a,t}$ holds. 

We next show $\mathbb{E}_{P_0}\left[\Psi^2_{a, t}\mid\mathcal{F}_{t-1}\right] - 1\xrightarrow{\mathrm{a.s}} 0$. This result is a direct consequence of Lemma~\ref{lem:almost_sure_nuisance}. 
\begin{lemma}
\label{lem:almost_conv}
For any $P_0\in\mathcal{P}$, we have $\mathbb{E}_{P_0}\left[\Psi^2_{a, t}\mid\mathcal{F}_{t-1}\right] - 1\xrightarrow{\mathrm{a.s}} 0$ as $t\to \infty$. 
\end{lemma}

\begin{proof}[Proof of Lemma~\ref{lem:almost_conv}]
For all $b\in[K]$, define 
\begin{align*}
    &\psi_{b, t} \coloneqq \left(\frac{\mathbbm{1}[A_t = b]\big(Y_{b, t}- \widehat{\mu}_{b, t}\big)}{\widehat{w}^{\mathrm{GNA}}_{b, t}} + \widehat{\mu}_{b, t}\right)/\sqrt{V(b)}. 
\end{align*}

We have
\begin{align*}
    &V(a)\mathbb{E}_{P_0}\left[\Psi^2_{a, t}\mid\mathcal{F}_{t-1}\right]= \mathbb{E}_{P_0}\left[\left(\psi_{a^\star(P_0), t} - \psi_{a, t}- \Delta_a(P_0)\right)^2\Big| \mathcal{F}_{t-1}\right]
    \\
    &= \mathbb{E}_{P_0}\Bigg[\frac{\mathbbm{1}[A_{t} = 1]\big(Y_{a^\star(P_0), t}- \widehat{\mu}_{a^\star(P_0), t}\big)^2}{\big(\widehat{w}^{\mathrm{GNA}}_{a^\star(P_0), t}\big)^2} + \frac{\mathbbm{1}[A_{t} = 2]\big(Y_{a, t}- \widehat{\mu}_{a, t}\big)^2}{\big(\widehat{w}^{\mathrm{GNA}}_{a, t}\big)^2}
    \\
    &\ + 2\left(\frac{\mathbbm{1}[A_{t}  = 1]\big(Y_{a^\star(P_0), t}- \widehat{\mu}_{a^\star(P_0), t}\big)}{\widehat{w}^{\mathrm{GNA}}_{a^\star(P_0), t}} - \frac{\mathbbm{1}[A_t = a]\big(Y_{a, t}- \widehat{\mu}_{a, t}\big)}{\widehat{w}^{\mathrm{GNA}}_{a, t}}\right)\\
    &\ \ \ \ \ \ \ \ \ \ \ \ \ \ \ \ \ \ \ \ \ \ \ \ \ \ \ \ \ \ \ \ \ \ \ \ \ \ \ \ \ \ \ \ \ \ \ \ \cdot \Big( \widehat{\mu}_{a^\star(P_0), t} - \widehat{\mu}_{a, t} - (\mu_{a^\star(P_0)}(P_0) - \mu_a(P_0))\Big)
    \\
    &\ + \left( \big(\widehat{\mu}_{a^\star(P_0), t} - \widehat{\mu}_{a, t}\big) - (\mu_{a^\star(P_0)}(P_0) - \mu_a(P_0))\right)^2 |  \mathcal{F}_{t-1}\Bigg]
    \\
    &= \sum_{a\in\{1, 0\}}\mathbb{E}_{P_0}\left[\frac{\big(Y_{a,t}- \widehat{\mu}_{a, t}\big)^2}{\big(\widehat{w}^{\mathrm{GNA}}_{a,t}\big)^2}\mid  \mathcal{F}_{t-1}\right]-  \mathbb{E}_{P_0}\left[\left(\big(\widehat{\mu}_{a^\star(P_0), t} - \widehat{\mu}_{a, t}\big) - \big(\mu_{a^\star(P_0)}(P_0)- \mu_a(P_0)\big)\right)^2|  \mathcal{F}_{t-1}\right]. \label{eq:check4}
\end{align*}
Here, for $a\in\{1, 0\}$, the followings hold:
\begin{align*}
&\mathbb{E}_{P_0}\Bigg[\frac{\mathbbm{1}[A_{t} = a]\big(Y_{a,t}- \widehat{\mu}_{a, t}\big)^2}{\big(\widehat{w}^{\mathrm{GNA}}_{a,t}\big)^2}\mid  \mathcal{F}_{t-1}\Bigg]\\
&=\mathbb{E}_{P_0}\Bigg[\frac{\big(Y_{a,t}- \widehat{\mu}_{a, t}\big)^2}{\widehat{w}^{\mathrm{GNA}}_{a,t}}\mid  \mathcal{F}_{t-1}\Bigg]\\
&=\frac{\mathbb{E}_{P_0}[(Y_{a,t})^2] - 2\mu_a(P_0)\widehat{\mu}_{a, t}+ (\widehat{\mu}_{a, t})^2}{\widehat{w}^{\mathrm{GNA}}_{a,t}}\\
&=\frac{\mathbb{E}_{P_0}[(Y_{a,t})^2] - (\mu_a(P_0))^2 + (\mu_a(P_0) - \widehat{\mu}_{a, t})^2}{\widehat{w}^{\mathrm{GNA}}_{a,t}}\\
&=\frac{\sigma^2_a + (\mu_a(P_0) - \widehat{\mu}_{a, t})^2}{\widehat{w}^{\mathrm{GNA}}_{a,t}},\quad \mathrm{and}\\
&\mathbb{E}_{P_0}\Bigg[\frac{\mathbbm{1}[A_{t} = a]\big(Y_{a,t}- \widehat{\mu}_{a, t}\big)^2}{\big(\widehat{w}^{\mathrm{GNA}}_{a^\star(P_0), t}\big)^2}\frac{\mathbbm{1}[A_{t} = a]\big(Y_{a,t}- \widehat{\mu}_{a, t}\big)^2}{\big(\widehat{w}^{\mathrm{GNA}}_{a, t}\big)^2}\mid  \mathcal{F}_{t-1}\Bigg] = 0,
\end{align*}
where we used $\mathbb{E}_{P_0}[(Y_{a,t})^2] - \mu_a(P_0)^2 = \sigma^2_a$.
Therefore, the following holds:
\begin{align*}
& \mathbb{E}_{P_0}\left[\frac{\big(Y_{a^\star(P_0), t}- \widehat{\mu}_{a^\star(P_0), t}\big)^2}{\widehat{w}^{\mathrm{GNA}}_{a^\star(P_0), t}}\mid  \mathcal{F}_{t-1}\right]+ \mathbb{E}_{P_0}\left[\frac{\big(Y_{a, t}- \widehat{\mu}_{a, t}\big)^2}{\widehat{w}^{\mathrm{GNA}}_{a, t}}\mid  \mathcal{F}_{t-1}\right]\\
&\ \ \ \ \ \ -  \mathbb{E}_{P_0}\left[\left(\big(\widehat{\mu}_{a^\star(P_0), t} - \widehat{\mu}_{a, t}\big) - (\mu_{a^\star(P_0)}(P_0) - \mu_a(P_0))\right)^2|  \mathcal{F}_{t-1}\right]
\\
&= \mathbb{E}_{P_0}\left[\frac{\sigma^2_{a^\star(P_0)} + (\mu_{a^\star(P_0)}(P_0) - \widehat{\mu}_{a^\star(P_0), t})^2}{\widehat{w}^{\mathrm{GNA}}_{a^\star(P_0), t}}\right] +  \mathbb{E}_{P_0}\left[\frac{\sigma^2_{a} + (\mu_a(P_0) - \widehat{\mu}_{a, t})^2}{\widehat{w}^{\mathrm{GNA}}_{a, t}}\right]\\
&\ \ \ \ \ \  -  \mathbb{E}_{P_0}\left[\left(\big(\widehat{\mu}_{a^\star(P_0), t} - \widehat{\mu}_{a, t}\big) - (\mu_{a^\star(P_0)}(P_0) - \mu_a(P_0))\right)^2\right].
\end{align*}

From Lemma~\ref{lem:almost_sure_nuisance}, because $\widehat{\mu}_{b, t}\xrightarrow{\mathrm{a.s.}} \mu_b$ and $\widehat{w}^{\mathrm{GNA}}_{b, t}\xrightarrow{\mathrm{a.s.}} w^{\mathrm{GNA}}_{b}$ hold, we have
\begin{align*}
&\Bigg|\left(\frac{\sigma^2_{a^\star(P_0)} + (\mu_{a^\star(P_0)} - \widehat{\mu}_{a^\star(P_0), t})^2}{\widehat{w}^{\mathrm{GNA}}_{a^\star(P_0), t}}\right)+ \left(\frac{\sigma^2_{a} + (\mu_a(P_0) - \widehat{\mu}_{a, t})^2}{\widehat{w}^{\mathrm{GNA}}_{a, t}}\right)\\
&\ \ \ \ \ \ \ \ \ \ \ \ \ \ \ \ \ \ \ \ \ - \left(\big(\widehat{\mu}_{a^\star(P_0), t} - \widehat{\mu}_{a, t}\big) - \big(\mu_{a^\star(P_0)}(P_0) - \mu_a(P_0)\big)\right)^2\\
&\ \ \ \ \ \ \ \ \ \ \ \ \ \ \ \ \ \ \ \ \ -  \left(\frac{\sigma^2_{a^\star(P_0)}}{w^{\mathrm{GNA}}_{a^\star(P_0)}} + \frac{\sigma^2_a}{w^{\mathrm{GNA}}_{a}} + \left(\big(\mu_{a^\star(P_0)}(P_0) - \mu_a(P_0)\big) - (\mu_{a^\star(P_0)}(P_0) - \mu_a(P_0))\right)^2\right)\Bigg|
\\
&\leq \left|\frac{\sigma^2_{a^\star(P_0)}}{\widehat{w}^{\mathrm{GNA}}_{a^\star(P_0), t}} - \frac{\sigma^2_{a^\star(P_0)}}{w^{\mathrm{GNA}}_{a^\star(P_0)}}\right| +  \lim_{t\to\infty}\left| \frac{\sigma^2_{a}}{\widehat{w}^{\mathrm{GNA}}_{a, t}} - \frac{\sigma^2_{a}}{w^{\mathrm{GNA}}_{a}}\right|\\
&\ \ + \lim_{t\to\infty}\frac{(\mu_{a^\star(P_0)}(P_0) - \widehat{\mu}_{a^\star(P_0), t})^2}{\widehat{w}^{\mathrm{GNA}}_{a^\star(P_0), t}} + \lim_{t\to\infty}\frac{ (\mu_a(P_0) - \widehat{\mu}_{a, t})^2}{\widehat{w}^{\mathrm{GNA}}_{a, t}}\\
&\ \ + \lim_{t\to\infty}\big|\left(\big(\widehat{\mu}_{a^\star(P_0), t} - \widehat{\mu}_{a, t}\big) - \big(\mu_{a^\star(P_0)}(P_0) - \mu_a(P_0)\big)\right)^2\\
&\ \ \ \ \ \ \ \ \ \ \ \ \ \ \ \ \ \ \ \ \ \ - \left(\big(\mu_{a^\star(P_0)}(P_0) - \mu_a(P_0)\big) - \big(\mu_{a^\star(P_0)}(P_0) - \mu_a(P_0)\big)\right)^2\big|\\
&\xrightarrow{\mathrm{a.s.}} 0,
\end{align*}
as $T \to \infty$.
Therefore, we obtain 
\begin{align*}
&V(a)\mathbb{E}_{P_0}\left[\Psi^2_{a, t}\mid\mathcal{F}_{t-1}\right] - V(a)\\
&= \mathbb{E}_{P_0}\left[\frac{\sigma^2_{a^\star(P_0)} + (\mu_{a^\star(P_0)}(P_0) - \widehat{\mu}_{a^\star(P_0), t})^2}{\widehat{w}^{\mathrm{GNA}}_{a^\star(P_0), t}}\mid\mathcal{F}_{t-1}\right] +  \mathbb{E}_{P_0}\left[\frac{\sigma^2_a + (\mu_a(P_0) - \widehat{\mu}_{a, t})^2}{\widehat{w}^{\mathrm{GNA}}_{a, t}}\mid\mathcal{F}_{t-1}\right]\\
&-  \mathbb{E}_{P_0}\left[\left(\big(\widehat{\mu}_{a^\star(P_0), t} - \widehat{\mu}_{a, t}\big) - \big(\mu_{a^\star(P_0)}(P_0) - \mu_a(P_0)\big)\right)^2\mid\mathcal{F}_{t-1}\right]\\
&- \mathbb{E}_{P_0}\Bigg[\frac{\sigma^2_{a^\star(P_0)}}{w^{\mathrm{GNA}}_{a^\star(P_0)}} + \frac{\sigma^2_{a}}{w^{\mathrm{GNA}}_{a}}+ \left(\big(\mu_{a^\star(P_0)}(P_0) - \mu_a(P_0)\big) - \big(\mu_{a^\star(P_0)}(P_0) - \mu_a(P_0)\big)\right)^2\mid\mathcal{F}_{t-1}\Bigg]\\
&\xrightarrow{\mathrm{a.s.}} 0,
\end{align*}
as $T \to \infty$.
\end{proof}

This lemma immediately yields Lemma~\ref{lem:cesaro}. 
This result is a variant of the Ces\`{a}ro lemma for a case with almost sure convergence. We show the proof, which is based on the proof of Lemma~10 in \citet{hadad2019}.
\begin{proof}[Proof of Lemma~\ref{lem:cesaro}]
Let $u_t$ be $u_t =  \mathbb{E}_{P_0}\left[\Psi^2_{a, t}\mid\mathcal{F}_{t-1}\right] - 1$. Note that $ \frac{1}{T}\sum^T_{t=1}\mathbb{E}_{P_0}\left[\Psi^2_{a, t}\mid\mathcal{F}_{t-1}\right] - 1 = \frac{1}{T}\sum^T_{t=1} u_t$. 

From the proof of Lemma~\ref{lem:almost_conv}, we can find that $u_t$ is a bounded random variable. Recall
\begin{align*}
    &V(a)\mathbb{E}_{P_0}\left[\Psi^2_{a, t}\mid\mathcal{F}_{t-1}\right]\\
    &=\mathbb{E}_{P_0}\left[\frac{\sigma^2_{a^\star(P_0)} + (\mu_{a^\star(P_0)}(P_0) - \widehat{\mu}_{a^\star(P_0), t})^2}{\widehat{w}^{\mathrm{GNA}}_{a^\star(P_0), t}}\mid\mathcal{F}_{t-1}\right]+  \mathbb{E}_{P_0}\left[\frac{\sigma^2_{a} + (\mu_a(P_0) - \widehat{\mu}_{a, t})^2}{\widehat{w}^{\mathrm{GNA}}_{a, t}}\mid\mathcal{F}_{t-1}\right]\\
    &\ \ \ -  \mathbb{E}_{P_0}\left[\left(\big(\widehat{\mu}_{a^\star(P_0), t} - \widehat{\mu}_{a, t}\big) - \big(\mu_{a^\star(P_0)}(P_0) - \mu_a(P_0)\big)\right)^2\mid\mathcal{F}_{t-1}\right].
\end{align*}
We assumed that $(\mu_{a^\star(P_0)}(P_0), \mu_a(P_0), \widehat{\mu}_{a^\star(P_0), t}, \widehat{\mu}_{a, t}, \widehat{w}^{\mathrm{GNA}}_{a^\star(P_0), t}, \widehat{w}^{\mathrm{GNA}}_{a, t})$ are all bounded random variables. 
Let $C$ be a constant independent of $T$ such that $|u_t| < C$ for all $t\in\mathbb{N}$. 

Almost-sure convergence of $u_t$ to zero as $t\to\infty$ implies that for all $\epsilon > 0$, there exists $t(\epsilon)$ such that $|u_t| < \epsilon$ for all $t\geq t(\epsilon)$ with probability one. Let $\mathcal{E}(\epsilon)$ denote the event in which this happens; that is, $\mathcal{E}(\epsilon) = \{ |u_t| < \epsilon\quad \forall\ t \geq t(\epsilon)\}$. Under this event, for $T > t(\epsilon)$, it holds that $ \frac{1}{T}\sum^T_{t=1}|u_t| \leq \frac{1}{T}\sum^{t(\epsilon)}_{t=1} C + \frac{1}{T}\sum^{T}_{t=t(\epsilon) + 1} \epsilon = \frac{1}{T}t(\epsilon) C + \epsilon$, 
where $\frac{1}{T}t(\epsilon) C \to 0$ as $T\to \infty$. Therefore, for all $\epsilon > 0$, there exists $t(\epsilon) > 0$ such that for all $T > t(\epsilon)$, $\frac{1}{T}\sum^T_{t=1}|u_t| < \epsilon$ holds with probability one. 
\end{proof}